\documentclass[journal]{IEEEtran}
\usepackage{amsmath,amssymb,amsthm,mathtools}
\usepackage{cite}
\usepackage{graphicx}
\usepackage{booktabs}
\usepackage{xcolor}
\usepackage{colortbl}
\usepackage{hyperref}
\usepackage{enumitem}
\usepackage{amsmath,amsfonts}
\usepackage{algorithmic}
\usepackage{multirow} 
\usepackage{algorithm}
\usepackage{array}
\usepackage[caption=false,font=normalsize,labelfont=sf,textfont=sf]{subfig}
\usepackage{algorithm}   
\usepackage{textcomp}
\usepackage{stfloats}
\usepackage{url}
\usepackage{verbatim}
\usepackage{cite}
\usepackage{anyfontsize}
\usepackage{lineno}
\usepackage{amsthm}
\usepackage{comment}
\newtheorem{definition}{Definition}[section]
\newtheorem{lemma}[definition]{Lemma}
\newtheorem{proposition}[definition]{Proposition}
\newtheorem{corollary}[definition]{Corollary}

\theoremstyle{remark}
\newtheorem*{remark}{Remark}

\newcommand{\Ystar}{Y^{\star}}
\newcommand{\Istar}{I^{\star}}
\newcommand{\Yhat}{\hat{Y}}
\newcommand{\Ihat}{\hat{I}}
\newcommand{\Ttheta}{T_{\theta}}
\newcommand{\mtheta}{m_{\theta}}
\newcommand{\rtheta}{r_{\theta}}
\newcommand{\clip}{\operatorname{clip}_{[0,1]}}
\newcommand{\supp}{\operatorname{supp}}
\newcommand{\calT}{\mathcal{T}}
\newcommand{\calL}{\mathcal{L}}
\newcommand{\calN}{\mathcal{N}}
\newcommand{\calS}{\mathcal{S}}

\begin{document}
\title{AutoLumNet: Monotone Optimal Transport for Single-Shot Exposure Correction}
\author{%
Airin Akter Tania$^{*}$, Md Raihan Khan$^{*}$, and Mohiuddin Ahmad,~\IEEEmembership{Senior Member,~IEEE}\\
Department of Electrical and Electronic Engineering (EEE)\\
Khulna University of Engineering \& Technology (KUET), Bangladesh\\
\{airinaktertania, kraihan918\}@gmail.com, ahmad@eee.kuet.ac.bd\\
$^{*}$Equal contribution%
}


\markboth{IEEE TRANSACTIONS ON PATTERN ANALYSIS AND MACHINE INTELLIGENCE}%
{Shell \MakeLowercase{\textit{et al.}}: A Sample Article Using IEEEtran.cls for IEEE Journals}


\maketitle
\begin{abstract}
Single-shot exposure correction aims to map an arbitrarily degraded
image---whether under-exposed, over-exposed, or a spatial mixture of
both---to a well-exposed output from a single capture.  We present
AutoLumNet, a framework that decomposes this task into a global
monotone tone curve and a bounded local residual, making the global
component the locus of formal guarantees.  The tone curve is
parameterized as the normalized cumulative integral of a strictly
positive density, ensuring strict monotonicity by construction rather
than by penalty.  We prove that this parameterization (i)~preserves
the pairwise luminance ordering of all pixels and all spatial extrema
unconditionally, and (ii)~is dense in the space of valid tone
corrections, containing the one-dimensional optimal-transport map
from the input to any target luminance distribution.  A
differentiable sorted-sample Wasserstein-2 objective drives the
learned curve toward the OT optimum during training.  Spatially
varying effects that the global map provably cannot address---local
shading, chrominance shifts, and clipped-region restoration---are
handled by a bounded residual decoder with dual-branch convex fusion,
for which we provide an explicit sufficient condition for local order
preservation.  Experiments on five benchmarks (MSEC, SICE, LCDP,
LOL-v1, LOL-v2-real) show that AutoLumNet achieves state-of-the-art
PSNR and SSIM across both under- and over-exposure regimes at
11.2\,ms per frame, and generalizes zero-shot to pure low-light
benchmarks without retraining.  To our knowledge, AutoLumNet is the
first exposure-correction method to unite structural monotonicity,
optimal-transport optimality, and bounded local adaptivity within a
single trainable architecture. 
 \href{https://github.com/kraihan/Autolumnet}{https://github.com/kraihan/Autolumnet}

\end{abstract}



\section{Introduction}
\label{sec:intro}

\IEEEPARstart{I}{mages} captured in uncontrolled environments routinely
exhibit exposure errors. A single frame may contain deep shadows that
conceal structure in dark regions and, simultaneously, saturated
highlights where bright content has been clipped beyond recovery. These
degradations arise because the dynamic range of a digital sensor is far
narrower than that of a natural scene: regions outside the sensor's
operating window are compressed toward zero or one, discarding tonal
information. Exposure errors not only diminish perceptual quality but
also degrade downstream vision systems for detection, recognition, and
tracking, which are typically trained on well-exposed imagery. While
auto-exposure control and high-dynamic-range sensing partially mitigate
the problem, they introduce their own artifacts---motion blur, noise
amplification, ghosting from bracket misalignment---and are ill-suited
to dynamic scenes or resource-constrained mobile capture, where only a
single shot is available.

\textbf{Two disjoint research traditions.} The literature has largely addressed exposure degradation from two non-overlapping directions. The first, \emph{low-light image enhancement} (LLIE), assumes the input is globally under-exposed and seeks to brighten it. The dominant approaches include LLNet~\cite{LLNet}, Retinex-based  decomposition RetinexNet~\cite{RetinexNet}, KinD~\cite{KinD}, URetinex-Net~\cite{URetinex}, Retinexformer~\cite{retinexformer}), learnable tone curves (Zero-DCE~\cite{ZeroDCE}, EnlightenGAN~\cite{EnlightenGAN}), and more recently state-space backbones~\cite{RetinexMamba}, diffusion priors~\cite{RetiDiff}, and alternative color spaces~\cite{CIDNet}. Despite their sophistication, these methods share a defining assumption: the degradation is a \emph{monotone darkening}, and enhancement is a \emph{monotone brightening}. Applied to an over-exposed input, a model trained to brighten only worsens the saturation.

\begin{figure}[!t]
  \centering
  \includegraphics[width=\columnwidth]{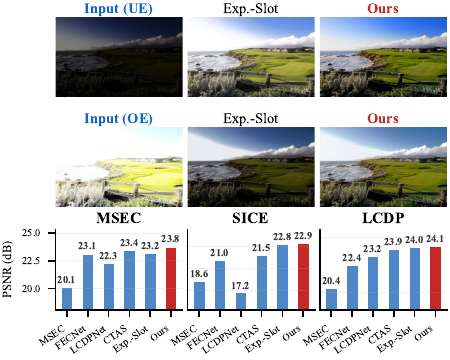}
  \vspace{-16pt}
\caption{\textbf{Single-shot exposure correction with AutoLumNet.}
\emph{Top two rows:} the same SICE scene captured under
under-exposure (UE) and over-exposure (OE).
Exposure-Slot~\cite{expslot}, the current multi-exposure SOTA,
produces acceptable but imperfect results in both directions;
AutoLumNet recovers colour and detail more faithfully.
\emph{Bottom row:} average PSNR on three standard benchmarks.
AutoLumNet achieves the highest PSNR on all three datasets among
the compared multi-exposure methods, while being the only method
whose global tone correction carries formal monotonicity and
optimal-transport guarantees (Sec.~\ref{sec:method}).}
\vspace{-20pt}
\label{fig:teaser}
\end{figure}

The second tradition, \emph{multi-exposure fusion} (MEF), combines
several differently exposed captures of the same
scene~\cite{Mertens2007,Ma2017SPDMEF,DeepFuse,Xu2020U2Fusion,MEFGAN,EMEF}.
MEF can, in principle, recover both shadow and highlight detail, since
the information missing from one exposure is present in another. Its
premise, however, is a fundamental obstacle: multiple spatially aligned
exposures of a static scene are rarely available in practice.

\textbf{The gap.}
Between these traditions lies a setting that neither addresses well: a
\emph{single} image that contains both under- and over-exposed regions.
Recent methods that explicitly target both
directions~\cite{afifi2021msec,enc,MMHT} confirm that this
bidirectional regime is both practically important and substantially
harder than either sub-problem alone. The difficulty is structural:
under- and over-exposure demand \emph{opposite} corrections, and a
single feed-forward mapping that must both brighten and darken places
conflicting demands on its features, tending to produce halos, color
shifts, or detail loss at exposure boundaries. Furthermore, almost all
existing methods are evaluated purely empirically: they offer no
guarantee that the enhanced image preserves the relative ordering of
scene luminances, a property whose violation manifests directly as the
halos and tonal reversals that plague enhancement outputs.

\textbf{Our approach.}
We study \emph{single-shot exposure correction} (SEC): the task of
mapping one degraded image to an output on the manifold of naturally
exposed photographs, regardless of whether its degradations are
under-exposure, over-exposure, or a spatial mixture of both. Rather than
treating correction as a monolithic learned mapping, we decompose it
into a \emph{global} tonal component and a \emph{local} residual
component, and we make the global component the locus of formal
guarantees. Concretely, we model the global luminance correction as a
strictly increasing tone map that is monotone \emph{by construction}---a
normalized cumulative integral of a strictly positive density---so that
order preservation is structural rather than the hoped-for outcome of
optimization.

\textbf{Contributions.}
Our contributions are as follows:
\begin{itemize}
    \item We formalize \emph{single-shot exposure correction} through an
    honest degradation model that explicitly distinguishes
    mathematically recoverable regions, where the tone map is invertible,
    from clipped regions, where recovery is necessarily prior-based
    estimation rather than inversion. This delineation makes precise what
    any single-shot method can and cannot guarantee.

    \item We introduce a \emph{monotone-by-construction} global tone map
    and prove (i)~its strict monotonicity, (ii)~exact preservation of the
    global luminance ordering and of all spatial extrema, and (iii)~that
    the parameterized family is dense in the space of valid tone curves
    and contains the one-dimensional optimal-transport map from the input
    to the target luminance distribution. The third result justifies an
    optimal-transport alignment objective whose minimizer is the correct
    global correction.

    \item We complement the global map with a bounded local residual and
    a convex dual-branch fusion, both of which carry explicit structural
    properties (a hard magnitude bound and a no-extrapolation guarantee),
    and we state precisely the conditional---rather than
    unconditional---sense in which local order is preserved.

    \item We demonstrate competitive performance against representative
    single-image and multi-exposure baselines on standard
    exposure-correction benchmarks, showing that the proposed framework
    attains accuracy comparable to recent methods while providing formal
    guarantees that prior work does not.
\end{itemize}

\noindent
The remainder of this paper is organized as follows.
Section~\ref{sec:related} reviews related work in low-light enhancement,
multi-exposure fusion, and optimal transport for imaging.
Section~\ref{sec:method} develops the method and its theoretical
guarantees. Section~\ref{sec:experiments} reports experiments and
ablations, and Section~\ref{sec:conclusion} concludes.

 
\section{Related Work}
\label{sec:related}

We situate our work among four related lines: low-light image enhancement (\S\ref{ssec:llie}), multi-exposure image fusion (\S\ref{ssec:mef}), single-shot
bidirectional correction (\S\ref{ssec:sec}), and optimal
transport in imaging (\S\ref{ssec:ot}).
 \vspace{-12pt} 
\subsection{Low-Light Image Enhancement}
\label{ssec:llie}

Low-light image enhancement (LLIE) recovers well-exposed imagery
from under-illuminated captures.  The dominant paradigm derives from
Retinex theory, decomposing an image into reflectance and
illumination components and restoring each
separately~\cite{RetinexNet,KinD,URetinex,URetinexNetPP,retinexformer,RetinexFormerPlus,M2Retinexformer}.
State-space backbones~\cite{RetinexMamba,MambaLLIE} and diffusion
priors~\cite{DiffRetinexPP,RetiDiff,QuadPrior,QuadPriorPP} have
recently extended this family.
A parallel line of work eschews explicit decomposition in favor of
learnable tone curves: Zero-DCE~\cite{ZeroDCE} and its efficient
variant~\cite{zerodcepp} estimate per-pixel brightening curves
without paired supervision, characterizing the curve family in terms
of monotonicity and differentiability;
SCI~\cite{sci,SCIPP} achieves real-time enhancement via
self-calibrated illumination.
Other directions include conditional normalizing
flows~\cite{Wang2022llflow}, trainable color
spaces~\cite{CIDNet}, adversarial training without paired
data~\cite{EnlightenGAN}, and spatial--frequency
architectures~\cite{StarIR}.

\textit{Limitation common to all LLIE methods.}\;
By construction, these methods assume a monotone darkening
degradation and apply a monotone brightening correction.  When
presented with over-exposed or mixed-exposure inputs---where
highlights must be \emph{attenuated}---their brightening bias
amplifies the saturation rather than correcting it.  AutoLumNet
operates in the strictly more general single-shot exposure
correction (SEC) regime, handling under-, over-, and mixed-exposure
images with a single model.

\vspace{-9pt}
\subsection{Multi-Exposure Image Fusion}
\label{ssec:mef}

Multi-exposure image fusion (MEF) sidesteps the single-image
limitation by combining several differently exposed captures of
the same scene.  Classical approaches weight pixels or
transform-domain coefficients by well-exposedness
criteria~\cite{Mertens2007, SmartClassroom, Ma2017SPDMEF,PerceptualMEF}, with
extensions to automatic exposure
compensation~\cite{AutoExpComp}.
Deep models learn the fusion mapping
directly~\cite{DeepFuse,Xu2020U2Fusion,MEFGAN}, with
EMEF~\cite{EMEF} adopting an ensemble strategy that fine-tunes a
style code at test time, and
Retinex-MEF~\cite{RetinexMEF} explicitly modelling glare effects
within an unsupervised Retinex framework.

\textit{Fundamental limitation.}\;
MEF methods require at least two spatially aligned exposures of
a static scene.  Handheld and dynamic capture violate the alignment
assumption, and most legacy or single-frame imagery offers no
bracket at all.  AutoLumNet retains MEF-style dual-branch
exposure reasoning---its shadow and highlight sub-networks
(\S\ref{sec:local})---while operating from a single shot.
\vspace{-12pt}
\subsection{Single-Shot Bidirectional Exposure Correction}
\label{ssec:sec}

The setting most directly addressed by AutoLumNet is single-shot
exposure correction (SEC): mapping one degraded image to its
well-exposed counterpart regardless of whether the degradation is
under-exposure, over-exposure, or a spatial mixture of both.

Afifi~\textit{et al.}~\cite{afifi2021msec} established the first SEC
benchmark (MSEC) with a coarse-to-fine correction network; the
companion SICE dataset~\cite{cai2018sice} provides multi-exposure
sequences from which single-shot pairs are derived.
A key challenge in SEC is the conflicting optimization between
under- and over-exposure corrections within a single network.
ENC~\cite{enc} and ECLNet~\cite{eclnet} address this by projecting
features to an exposure-invariant space, while
Huang~\textit{et al.}~\cite{ERL} correlate under- and over-exposed
samples across the batch dimension.
FECNet~\cite{fecnet} decomposes correction in the frequency domain,
LCDPNet~\cite{lcdpnet} exploits local color distributions as a
spatial prior, and IAT~\cite{iat} and CTAS~\cite{ctas} target
real-time inference via lightweight attention and learnable 3D
lookup tables, respectively.
More recent region-aware methods include
RECNet~\cite{recnet} with exposure-contrastive regularization,
CLIER~\cite{CLIER} with CLIP-guided refinement for extreme
exposures, and
Exposure-Slot~\cite{expslot}, the current state of the art, which
uses hierarchical slot attention to progressively cluster and
correct features by exposure level.

\textit{Structural gap.}\;
All of the above methods are evaluated purely empirically.  None
provides a structural guarantee that the enhanced image preserves
the relative ordering of scene luminances---a property whose
violation manifests as the halo artifacts, tonal reversals, and
detail loss that remain common failure modes.
The curve-estimation family (Zero-DCE and its variants) enforces
per-pixel monotonicity through a soft penalty, but as we show in
Remark~\ref{rem:why-global}, per-pixel monotone maps do not imply
spatial order preservation.
ENC and ECLNet project to an ``invariant'' feature space, but this
invariance is a regularization strategy rather than a provable
property.
AutoLumNet fills this gap: its monotone-by-construction global
tone map guarantees strict monotonicity, global pixel-order
preservation, and spatial-extremum preservation
(Lemma~\ref{lem:strict-mono}, Proposition~\ref{prop:order}), while the
parameterized family is shown to contain the optimal-transport map
(Proposition~\ref{prop:ot-in-family}).
\vspace{-8pt}
 \vspace{-8pt}
\subsection{Optimal Transport in Imaging}
\label{ssec:ot}
Optimal transport (OT) has been applied to example-based color
transfer~\cite{Pitie2007,Rabin2014}, and histogram equalization can
itself be viewed as a special case of 1-D OT mapping the input
luminance CDF to a uniform target~\cite{Santambrogio2015}.  These
prior works, however, apply a \emph{fixed} empirical OT map computed
between pre-collected statistics.  They do not learn a parameterized
family of monotone maps, nor prove that such a family is dense in
the space of valid tone corrections.  Our contribution is precisely
this density result (Proposition~\ref{prop:ot-in-family}), connecting
the learned tone curve to the OT optimum via a differentiable
sorted-sample Wasserstein-2 objective.  To our knowledge, no prior
exposure-correction method combines an optimal-transport
interpretation of the global luminance correction with structural
order-preservation guarantees.

\section{Method}
\label{sec:method}

\noindent
We develop AutoLumNet, a framework for single-shot exposure correction that
rests on three design principles formalized in the subsections below:
\begin{enumerate}[label=(\roman*)]
    \item an honest degradation model that distinguishes between
          mathematically recoverable and unrecoverable image regions
          (\S\ref{sec:prelim});
    \item a globally monotone tone map whose order-preservation property
          is provably guaranteed by construction (\S\ref{sec:tone}), and
          whose optimal setting is identified with the one-dimensional
          optimal-transport map (\S\ref{sec:ot}); and
    \item a spatially adaptive local correction via a bounded residual
          decoder with dual-branch convex fusion (\S\ref{sec:local}).
\end{enumerate}
\vspace{-16 pt}
\subsection{Preliminaries and Problem Formulation}
\label{sec:prelim}

\subsubsection{Notation}
\label{sec:notation}

Let $I \in [0,1]^{H \times W \times 3}$ denote a single input image with spatial
domain $\Omega = \{1,\dots,H\} \times \{1,\dots,W\}$ and let $p \in \Omega$
index pixels. We work primarily in the luminance--chrominance decomposition
of $I$: $Y \in [0,1]^{H \times W}$ is the luminance (Rec.~601: $Y = 0.299R +
0.587G + 0.114B$) and $C$ collects the two chrominance channels. The
ground-truth well-exposed image is $\Istar$ with corresponding luminance
$\Ystar$. The learnable map $F_\theta : I \mapsto \Ihat$ is the full
exposure-correction network. We use $\calT$ for the parameterized family of
tone maps and $\mathcal{P}^\star$ for the canonical well-exposed luminance
distribution. Table~\ref{tab:notation} provides a complete symbol reference.

\begin{table}[t]
\centering
\caption{Summary of notation used throughout Section~\ref{sec:method}.}
\label{tab:notation}
\small
\setlength{\tabcolsep}{4pt}
\begin{tabular}{@{}ll@{}}
\toprule
Symbol & Meaning \\
\midrule
$I,\,\Ihat,\,\Istar$       & Input, predicted, ground-truth image \\
$Y,\,\Yhat,\,\Ystar$       & Corresponding luminance channels \\
$C$                        & Chrominance channels of $I$ \\
$\Omega$                   & Pixel domain $\{1,\dots,H\}\times\{1,\dots,W\}$\\
$p,q$                      & Pixel indices \\
$\varphi_p$                & Pixelwise degradation map \\
$S$                        & Unsaturated pixel set (Definition~\ref{def:sat}) \\
$\Ttheta$                  & Global tone-correction map \\
$\mtheta$                  & Positive density parameterizing $\Ttheta$ \\
$w_k,\,c_k$                & Discrete bin masses and knot values \\
$\rtheta$                  & Bounded spatial residual, $|\rtheta|\le\rho$ \\
$\rho$                     & Residual budget (hyperparameter) \\
$\mu,\,\nu$                & Input and target luminance distributions \\
$F_\mu,\,F_\nu$            & Corresponding CDFs \\
$T^\star$                  & 1-D optimal-transport map \\
$\calT$                    & Parameterized monotone map family \\
$K$                        & Number of tone-curve bins \\
$\varepsilon$              & Positivity floor for $\mtheta$ \\
$\lambda_\bullet$          & Loss-weight hyperparameters \\
\bottomrule
\end{tabular}
\end{table}

 \begin{figure*}[t]
      \centering
      \includegraphics[width=\textwidth]{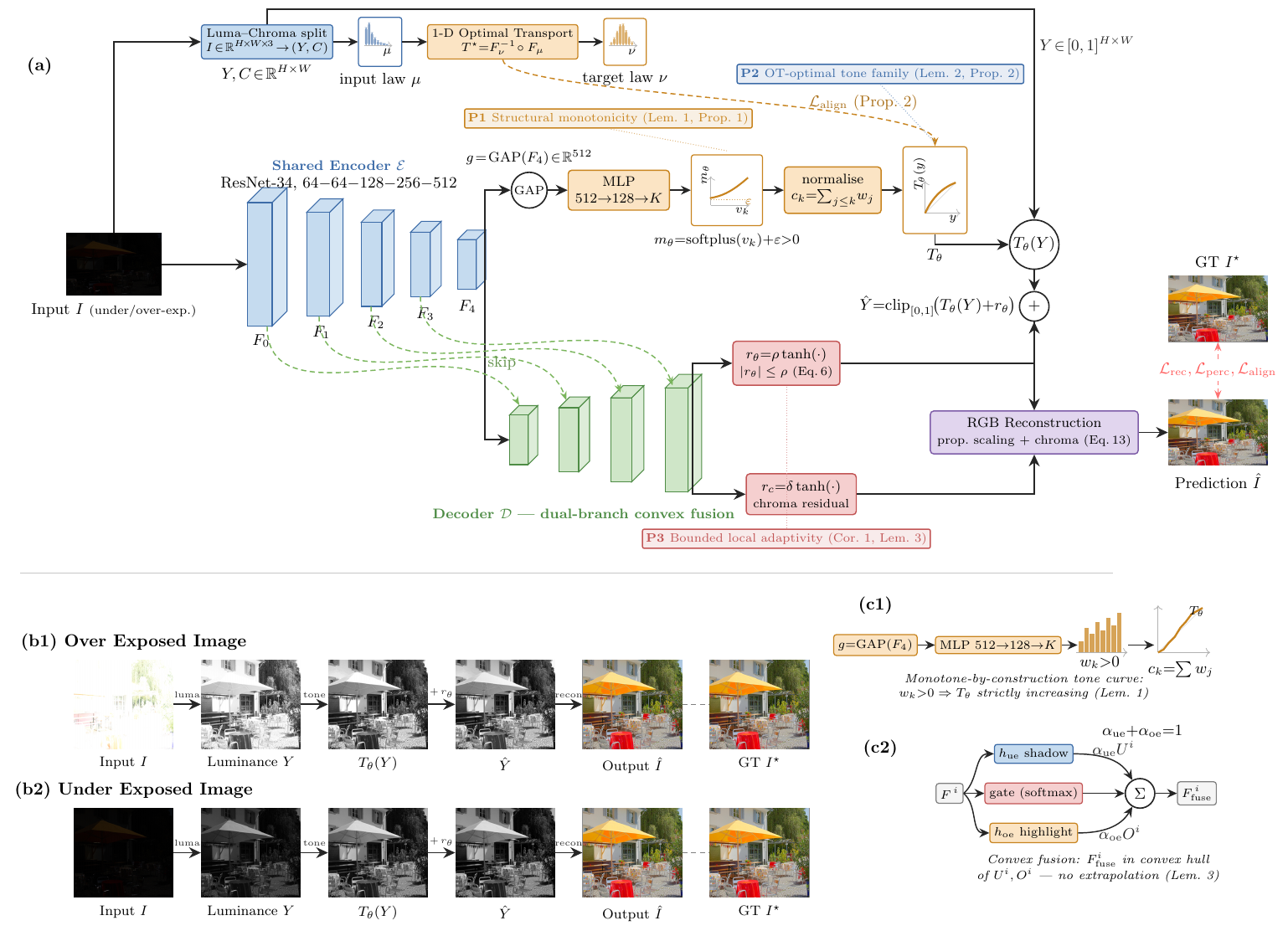}
      \vspace{-16pt}
      \caption{\textbf{Architecture of AutoLumNet.}
  \textbf{(a)} A shared Resnet encoder $\mathcal{E}$ feeds two paths.
  The \emph{global} path (top) realizes principles \textbf{P1} and
  \textbf{P2}: the input is split into luminance--chrominance, the input
  luminance law $\mu$ is aligned to a well-exposed target law $\nu$ by the
  one-dimensional optimal-transport map $T^\star{=}F_\nu^{-1}\!\circ F_\mu$
  ($\mathcal{L}_{\mathrm{align}}$), and a pooled descriptor
  $g{=}\mathrm{GAP}(F_4)$ is mapped by an MLP to bin logits $v_k$, passed
  through $\mathrm{softplus}$ with a positive floor to give strictly
  positive masses $m_\theta{=}\mathrm{softplus}(v_k){+}\varepsilon$,
  normalized and cumulatively summed into the strictly increasing tone
  curve $T_\theta$ (boxed insets), applied to the input luminance $Y$.
  The \emph{local} path (bottom) realizes principle \textbf{P3}: a
  dual-branch convex-fusion decoder $\mathcal{D}$ produces a bounded
  residual $r_\theta$ ($|r_\theta|\!\le\!\rho$) and a chroma residual
  $r_c$. The luma estimate
  $\hat Y{=}\mathrm{clip}_{[0,1]}(T_\theta(Y){+}r_\theta)$ is recomposed
  into RGB. Colored tags mark where each theoretical guarantee is
  enforced.
  \textbf{(b)} Progressive exposure correction across stages.
  \textbf{(c1)} Construction of the strictly increasing tone curve from
  positive bin masses (Lemma~1).
  \textbf{(c2)} Dual-branch convex fusion remains inside the convex hull
  of its two hypotheses (Lemma~3).}
  \vspace{-16pt}
      \label{fig:overview}
    \end{figure*}

\subsubsection{Problem Statement}
\label{sec:problem}

We formulate \emph{single-shot exposure correction} (SEC) as learning a
deterministic mapping $F_\theta : I \mapsto \Ihat$ such that $\Ihat$ lies on
the manifold of naturally exposed photographs, given only the single
degraded input $I$. This framing is deliberately agnostic to the nature of
the degradation: $I$ may be globally under-exposed, globally over-exposed, or
contain both shadow and highlight distortions simultaneously. 

\subsubsection{Degradation Model and Recoverability}
\label{sec:degradation}

We posit that the observed luminance arises from a well-exposed ideal
$\Ystar$ through a pixelwise strictly increasing tone map composed with
hard clipping:
\begin{equation}
\begin{aligned}
    Y(p) &= \clip\!\bigl(\varphi_p(\Ystar(p))\bigr), \\
    &\qquad \varphi_p \;\text{is continuous and strictly increasing.}
\end{aligned}
\label{eq:degradation}
\end{equation}

\begin{definition}[Unsaturated and saturated sets]
\label{def:sat}
The \emph{unsaturated set} is
$S := \{p \in \Omega : 0 < \varphi_p(\Ystar(p)) < 1\}$.
Its complement $\Omega \setminus S$ is the \emph{saturated set}, comprising
pixels whose true luminance was clipped to $0$ or $1$ by the camera sensor.
\end{definition}

\paragraph{Recoverable and unrecoverable regions.}
On $S$, the map $\varphi_p$ is invertible by strict monotonicity, so
$\Ystar(p) = \varphi_p^{-1}(Y(p))$ is in principle recoverable: correction
constitutes a genuine inversion. On $\Omega \setminus S$, the clip is
non-injective and its pre-image is an interval, so \emph{no deterministic
function inverts the saturation}. Recovery on the saturated set is therefore
a prior-based estimation problem, not an inversion problem. We state this
explicitly and design the residual decoder accordingly
(\S\ref{sec:local}); earlier formulations of this
work~\cite{anonymous_arxiv, CustomLowLight} made the unjustified claim that $h_\text{oe}$
inverted a saturating map, which we correct here.

\paragraph{Global--local factorization.}
We factor $\varphi_p = \varphi \circ \psi_p$, where $\varphi$ is a
\emph{global} monotone tonal component shared uniformly across all pixels and
$\psi_p$ encodes \emph{local} shading variation. Correcting $\varphi$
constitutes the global exposure problem and admits the provable guarantees
developed in \S\ref{sec:tone}--\S\ref{sec:ot}. Correcting $\psi_p$ and
restoring content on $\Omega \setminus S$ constitutes the local problem, handled by the bounded residual in \S\ref{sec:local}, for which we make
only conditional, not unconditional, claims.

\subsection{Overview of AutoLumNet}
\label{sec:overview}

\subsubsection{Design Principles}
\label{sec:principles}

The factorization $\varphi_p = \varphi \circ \psi_p$ motivates a
corresponding architectural decomposition. We design AutoLumNet around three
interlocking principles:

\smallskip
\noindent\textbf{P1 (Structural monotonicity).}
The global correction $\Ttheta$ is parameterized so that strict monotonicity
is guaranteed by construction, not enforced by a penalty. This eliminates
halo artifacts and preserves the spatial structure of the scene without
relying on optimization to satisfy what should be a hard constraint.

\smallskip
\noindent\textbf{P2 (Optimality of the tone family).}
The parameterized family $\calT$ is shown to contain the 1-D
optimal-transport map from input luminance distribution $\mu$ to the
well-exposed target distribution $\nu$. This justifies both the choice of
distribution-matching objective and the form of $\calT$.

\smallskip
\noindent\textbf{P3 (Bounded local adaptivity).}
A spatially varying residual $\rtheta$ with $|\rtheta| \le \rho$ provides
local correction that $\Ttheta$ alone cannot supply—including restoration of
clipped content under a learned prior—while a conditional theorem
characterizes precisely when local order is preserved.

\subsubsection{Architectural Overview}
\label{sec:arch_overview}

Figure~\ref{fig:overview} illustrates the full forward pass.
Given input $I$, a shared encoder $\mathcal{E}$ extracts a multi-scale
feature pyramid $\{F^0, F^1, F^2, F^3, F^4\}$ at successively finer-to-coarser
spatial resolutions. Two parallel computations operate on this pyramid:

\begin{enumerate}[label=(\arabic*)]
    \item \textbf{Global tone correction.} A global average-pooled descriptor
          $g = \operatorname{AvgPool}(F^4) \in \mathbb{R}^d$ is passed to the
          tone-curve head, which produces the monotone map $\Ttheta$.
          Applying $\Ttheta$ to the input luminance $Y$ yields the
          globally-corrected luminance $\Ttheta(Y)$.

    \item \textbf{Local residual correction.} A dual-branch decoder
          $\mathcal{D}$ processes $\{F^i\}$ and produces the bounded residual
          $\rtheta \in [-\rho, \rho]^{H \times W}$ and a chroma residual
          $r_c \in [-\delta, \delta]^{H \times W \times 2}$.
\end{enumerate}

These are composed via
\begin{equation}
    \Yhat(p) \;=\; \clip\!\bigl(\Ttheta(Y(p)) + \rtheta(p)\bigr),
    \label{eq:yhat_overview}
\end{equation}
and the final RGB output is obtained by transferring $\Yhat$ back to
the RGB domain while preserving chrominance ratios, followed by a small
chroma residual correction.

\begin{figure}[!t]
  \centering
  \includegraphics[width=\columnwidth]{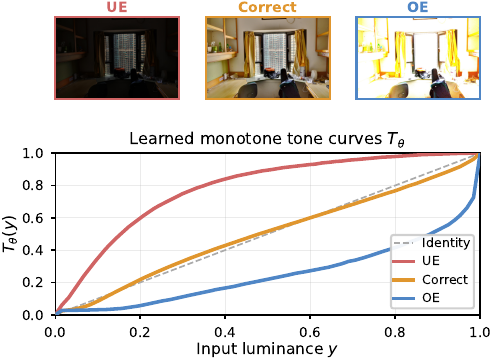}
  \vspace{-16pt}
  \caption{\textbf{Learned monotone tone curves $T_\theta$.}
For three exposures of the same scene (top), the model produces
image-adaptive, strictly increasing curves (bottom).
The under-exposed (UE) curve lies above the identity
(brightening), the correctly exposed curve is near-identity, and
the over-exposed (OE) curve lies below (darkening).
All three curves are monotone by construction
(Lemma~\ref{lem:strict-mono}), start at $(0,0)$ and end at $(1,1)$,
and the piecewise-linear staircase from $K{=}64$ bins is clearly
visible.  This validates the structural guarantee of
Proposition~\ref{prop:order}: a single global curve per image
preserves the spatial luminance ordering regardless of the
exposure direction.}
\vspace{-16pt}
\label{fig:tone_curves}
\end{figure}
\vspace{-8pt}
\subsection{Global Tone Correction via Monotone Optimal Transport}
\label{sec:tone}

\subsubsection{Monotone-by-Construction Parameterization}
\label{sec:tone_param}

We model the global luminance correction as a map
$\Ttheta : [0,1] \to [0,1]$. Rather than appending a post-hoc penalty for
monotonicity violations, we build strict monotonicity directly into the
parameterization. Let $\mtheta : [0,1] \to \mathbb{R}_{>0}$ be produced by
passing a neural network output through the softplus function and adding a
positive floor $\varepsilon > 0$, so that $\mtheta(t) \geq \varepsilon$ for
all $t$. Define the \emph{monotone tone map}:
\begin{equation}
    \Ttheta(y) \;=\;
    \frac{\displaystyle\int_0^{y} \mtheta(t)\,\mathrm{d}t}
         {\displaystyle\int_0^{1} \mtheta(t)\,\mathrm{d}t}.
    \label{eq:tone-cont}
\end{equation}

\paragraph{Discretization.}
In practice, we discretize~\eqref{eq:tone-cont} using $K$ uniform bins.
Specifically, we predict a vector of raw weights $\mathbf{v} \in \mathbb{R}^K$
from the network and form strictly positive masses
$w_k = \operatorname{softplus}(v_k) + \varepsilon > 0$, normalized so that
$\sum_{k=1}^{K} w_k = 1$. The discrete tone curve is the piecewise linear
function whose knots have abscissae $t_k = k/K$ and ordinates
$c_0 = 0$, $c_k = \sum_{j \leq k} w_j$, $c_K = 1$.  Piecewise linear
interpolation between strictly increasing knots is itself strictly increasing,
so the discretization preserves Lemma~\ref{lem:strict-mono} below.

\paragraph{Critical design choice.}
The weights $\{w_k\}$, and hence the tone curve $\Ttheta$, are predicted from
a \emph{global, image-level} descriptor obtained by global average pooling of
the deepest encoder feature: $g = \operatorname{AvgPool}(F^4) \in \mathbb{R}^d$.
Consequently, there is \emph{exactly one tone curve per image}, not one per
pixel. This design choice, which may appear restrictive, is precisely what
enables the spatial order-preservation guarantee of Proposition~\ref{prop:order}
below; per-pixel tone parameters provably do not suffice (Remark~\ref{rem:why-global}).

\subsubsection{Structural Guarantees}
\label{sec:tone_theory}

\begin{lemma}[Strict monotonicity]
\label{lem:strict-mono}
The map $\Ttheta$ defined in~\eqref{eq:tone-cont} is strictly increasing on
$[0,1]$, with $\Ttheta(0)=0$ and $\Ttheta(1)=1$.
\end{lemma}

\begin{proof}
Let $Z := \int_0^1 \mtheta(t)\,\mathrm{d}t > 0$.
For any $0 \leq y_1 < y_2 \leq 1$,
\[
    \Ttheta(y_2) - \Ttheta(y_1)
    \;=\; \frac{1}{Z}\int_{y_1}^{y_2} \mtheta(t)\,\mathrm{d}t
    \;\geq\; \frac{\varepsilon\,(y_2 - y_1)}{Z}
    \;>\; 0,
\]
since $\mtheta(t) \geq \varepsilon > 0$ by construction.
The endpoint identities $\Ttheta(0) = 0$ and $\Ttheta(1) = 1$ follow
immediately from the normalization $\sum_k w_k = 1$.
\end{proof}

\begin{proposition}[Global order and spatial-extremum preservation]
\label{prop:order}
Let $\Yhat(p) := \Ttheta\bigl(Y(p)\bigr)$ for all $p \in \Omega$, with
$\Ttheta$ as in Lemma~\ref{lem:strict-mono}.  Then the following hold.
\begin{enumerate}[label=(\roman*)]
    \item \emph{Global order preservation.} For all $p, q \in \Omega$,
    \[
\begin{aligned}
Y(p) < Y(q) &\;\Longleftrightarrow\; \widehat{Y}(p) < \widehat{Y}(q),\\
Y(p) = Y(q) &\;\Longleftrightarrow\; \widehat{Y}(p) = \widehat{Y}(q).
\end{aligned}
\]
    \item \emph{Spatial-extremum preservation.} For any neighbourhood
    $\calN(p) \subseteq \Omega$, $p$ is a local maximum (resp.\ minimum) of
    $Y$ on $\calN(p) \cup \{p\}$ if and only if $p$ is a local maximum
    (resp.\ minimum) of $\Yhat$ on the same neighbourhood.
\end{enumerate}
Consequently, $\Ttheta$ introduces no new spatial extrema and
destroys none.
\end{proposition}

\begin{proof}
Claim~(i) follows from Lemma~\ref{lem:strict-mono}: a strictly increasing
bijection and its inverse preserve and reflect strict inequalities and
equalities. For~(ii), $p$ is a local maximum of $Y$ on
$\calN(p) \cup \{p\}$ iff $Y(p) \geq Y(q)$ for all
$q \in \calN(p)$.  By~(i) each such inequality is preserved under $\Ttheta$,
so the condition holds iff $\Yhat(p) \geq \Yhat(q)$ for all
$q \in \calN(p)$.  The argument for local minima is identical.
\end{proof}

\begin{remark}
\label{rem:why-global}
Proposition~\ref{prop:order} necessitates that $\Ttheta$ not depend on the
pixel index $p$.  To see why, consider the per-pixel affine alternative
$T_p(y) = a(p)\,y + b(p)$ with $a(p) > 0$ for all $p$.  Even though each
$T_p$ is individually monotone, spatial order need not be preserved: set
$Y(p_1) = 0.2$, $a(p_1) = 3$ and $Y(p_2) = 0.3$, $a(p_2) = 0.5$.
Then $Y(p_1) < Y(p_2)$ yet $\Yhat(p_1) = 0.6 > 0.15 = \Yhat(p_2)$, so
the order is reversed by individually monotone per-pixel maps. Per-pixel
monotonicity is necessary but not sufficient for the spatial
order-preservation asserted in Proposition~\ref{prop:order}(ii).
\end{remark}

\subsection{Distribution Alignment via 1-D Optimal Transport}
\label{sec:ot}

\subsubsection{Optimality of the Parameterized Family}
\label{sec:ot_theory}

While Proposition~\ref{prop:order} characterizes the structural properties
of $\Ttheta$, it does not specify \emph{which} element of $\calT$ to select.
We now show that the theoretically optimal selection is precisely the
one-dimensional optimal-transport (OT) map from the input luminance
distribution $\mu$ to a well-exposed target distribution $\nu$, and that
this map is expressible within $\calT$.

Let $F_\mu$ and $F_\nu$ denote the cumulative distribution functions of $\mu$
and $\nu$, respectively.

\begin{lemma}[Monotone rearrangement with possibly atomic marginals]
\label{lem:ot-1d}
Let $\mu,\nu\in\mathcal{P}([0,1])$ with CDFs $F_\mu,F_\nu$ and quantile
$F_\nu^{-1}(u):=\inf\{t:F_\nu(t)\ge u\}$, and $c(x,y)=h(|x-y|)$ with $h$
strictly convex. \emph{(i)} If $\mu$ is atomless, $T^\star:=F_\nu^{-1}\!\circ
F_\mu$ is non-decreasing, satisfies $T^\star_{\#}\mu=\nu$, and is the
$\mu$-a.e.\ unique optimal map; if also $\nu$ is atomless and $\mu$ has a
positive density on interval support, $T^\star$ is strictly increasing on
$\supp(\mu)$. \emph{(ii)} If $\mu$ has an atom and $\nu$ is atomless, no
measurable map pushes $\mu$ to $\nu$; the comonotone coupling
$\gamma^\star=(F_\mu^{-1},F_\nu^{-1})_{\#}\mathrm{Leb}$ is optimal instead.
\end{lemma}
\begin{proof}
(i) Atomlessness makes $F_\mu$ continuous, so $(F_\mu)_{\#}\mu=\mathrm{Leb}$
and $(F_\nu^{-1})_{\#}\mathrm{Leb}=\nu$; composing gives $T^\star_{\#}\mu=\nu$.
Optimality and a.e.\ uniqueness for strictly convex 1-D costs are
classical~\cite[\S2.2]{Santambrogio2015}; strictness holds since $F_\mu$ is
then strictly increasing and $F_\nu^{-1}$ is so iff $\nu$ is atomless. (ii) A
measurable map sends atoms to atoms, so none attains the atomless $\nu$;
optimality of $\gamma^\star$ is classical~\cite[\S2.2]{Santambrogio2015}.
\end{proof}
\begin{remark}[Scope under clipping]
Under~\eqref{eq:degradation}, clipping places atoms of $\mu$ at $\{0,1\}$, so
case (ii) governs the full law: no tone curve transports $\mu$ to an atomless
target. The map-level optimum (i) is thus claimed only for the unsaturated
restriction $\mu_S:=\mathrm{law}(Y(p):p\in S)$, atomless whenever $Y^\star$ is.
Since $T_\theta$ is injective it preserves atoms, so spreading clipped mass
requires the residual $r_\theta$; training is unaffected, as the sorted-sample
loss~\eqref{eq:align} equals $W_2^2$ for arbitrary marginals.
\end{remark}

\begin{proposition}[The OT map lies in the parameterized family]
\label{prop:ot-in-family}
The family $\calT$ of maps defined by~\eqref{eq:tone-cont}, over all
admissible parameters $\theta$, is dense in the sup-norm in the set of
continuous non-decreasing surjections $[0,1]\to[0,1]$. Consequently, the
1-D OT map $T^\star$ of Lemma~\ref{lem:ot-1d}, whenever continuous (e.g.
$\nu$ has connected support), is approximable by elements of $\calT$ to
arbitrary sup-norm accuracy.
\end{proposition}
\begin{proof}
\emph{Exact weights.} For any $\hat w\in\Delta^{K-1}$ with $\hat w_k>0$,
pick $c>\varepsilon/\min_k\hat w_k$ and $v_k=\operatorname{softplus}^{-1}
(c\hat w_k-\varepsilon)$; then $\operatorname{softplus}(v_k)+\varepsilon
=c\hat w_k$, so after normalization the realized masses equal $\hat w$
exactly. Every strictly positive probability vector on $K$ bins is thus
attainable.
\emph{Strictly increasing $f$.} Set $\hat w_k=f(t_k)-f(t_{k-1})>0$; the
induced $T_\theta$ is the piecewise-linear interpolant of $f$ at knots
$t_k=k/K$. On each cell both $f$ and $T_\theta$ lie in
$[f(t_{k-1}),f(t_k)]$, so $\|f-T_\theta\|_\infty\le\max_k\bigl(f(t_k)
-f(t_{k-1})\bigr)\le\omega_f(1/K)\to0$ as $K\to\infty$.
\emph{Non-decreasing $f$.} Given $\eta>0$, take $K$ with $\omega_f(1/K)
\le\eta/2$ and $\hat w_k=(1-\tfrac\eta2)(f(t_k)-f(t_{k-1}))+\tfrac{\eta}
{2K}>0$; the knot ordinates satisfy $|c_k-f(t_k)|\le\eta/2$, whence
$\|f-T_\theta\|_\infty\le\eta$. Applying this to $T^\star$ (Lemma
\ref{lem:ot-1d}) gives the claim.
\end{proof}

\subsubsection{Learning the Tone Curve}
\label{sec:ot_learning}

Computing $F_\nu^{-1} \circ F_\mu$ analytically during training is awkward,
and in the supervised setting paired data are available.  We therefore
\emph{learn} $\Ttheta$ within $\calT$ by minimizing a differentiable
distributional surrogate.  Specifically, we use the debiased Sinkhorn
divergence:
\begin{equation}
    \calL_{\text{align}} \;=\;
    \calS_{\varepsilon}\!\bigl(\Ttheta{}_{\#}\hat{\mu},\;\hat{\nu}\bigr),
    \label{eq:align}
\end{equation}
where $\hat{\mu}$, $\hat{\nu}$ are empirical one-dimensional luminance
distributions and $\calS_\varepsilon$ is the Sinkhorn divergence with
entropic regularization $\varepsilon$.  As $\varepsilon \to 0$,
$\calS_\varepsilon \to W_2^2$ and $\calS_\varepsilon(\alpha,\alpha) = 0$.
In our implementation we use the equivalent sorted-sample quadratic loss,
\begin{equation}
    \calL_{\text{align}}
    \;=\;
    \frac{1}{N}\sum_{n=1}^{N}
    \bigl(\Yhat_{(n)} - \nu_{(n)}\bigr)^2,
    \label{eq:w2}
\end{equation}
where $\Yhat_{(1)} \leq \cdots \leq \Yhat_{(N)}$ are the sorted predicted
luminance values and $\nu_{(1)} \leq \cdots \leq \nu_{(N)}$ are quantiles
of $\hat{\nu}$ at the corresponding levels.  In one dimension,
\eqref{eq:w2} equals $W_2^2(\hat{\mu}, \hat{\nu})$ exactly, is
differentiable through the network via the gradients of the sort operation,
and is consistent with $T^\star$ by Proposition~\ref{prop:ot-in-family}.

\paragraph{Choice of target distribution.}
A fixed prior such as a truncated Gaussian $\calN(0.5, \sigma^2)|_{[0,1]}$
is a convenient default but constitutes a modeling choice rather than a
theorem.  Empirical luminance statistics of natural images are heavy-tailed
and scene-dependent. In practice, we estimate $\mathcal{P}^\star$
from the empirical luminance distribution of well-exposed images in the
training set and optionally condition it on the scene type through the
encoder; the truncated Gaussian is retained only as a fallback for
scene types underrepresented in training data.

\subsection{Local Correction and Dual-Branch Fusion}
\label{sec:local}

\subsubsection{Bounded Spatial Residual}
\label{sec:residual}

The global map $\Ttheta$ alone cannot address three aspects of the correction
that require spatial adaptivity: local shading variation $\psi_p$,
restoration of content on the clipped set $\Omega \setminus S$, and
chrominance adjustment. Furthermore, Proposition~\ref{prop:order} specifically
\emph{forbids} spatial reordering by design, which limits the ability of
$\Ttheta$ to correct localized exposure anomalies. We address both
limitations by adding a bounded spatial residual to the tonal correction.
The predicted luminance is defined as:
\begin{equation}
    \Yhat(p) \;=\;
    \clip\!\Bigl(\,\Ttheta\bigl(Y(p)\bigr) + \rtheta(p)\,\Bigr),
    \qquad
    \bigl|\,\rtheta(p)\,\bigr| \;\leq\; \rho,
    \label{eq:compose}
\end{equation}
where $\rtheta : \Omega \to \mathbb{R}$ is produced by the U-Net decoder
with a $\rho\tanh(\cdot)$ output head, and $\rho > 0$ is a hyperparameter
controlling the magnitude of permitted local deviation.   

\subsubsection{Conditional Order Preservation}
\label{sec:conditional_order}

\begin{corollary}[Conditional local order preservation]
\label{cor:local-order}
For neighboring pixels $p, q \in \Omega$ satisfying $Y(p) < Y(q)$, the
composed mapping in~\eqref{eq:compose} preserves their luminance order if
\begin{equation}
    \Ttheta\!\bigl(Y(q)\bigr) - \Ttheta\!\bigl(Y(p)\bigr)
    \;>\;
    \bigl|\,\rtheta(p) - \rtheta(q)\,\bigr|.
    \label{eq:order_condition}
\end{equation}
\end{corollary}

\begin{proof}
Subtracting $\Yhat(p) = \Ttheta(Y(p)) + \rtheta(p)$ from
$\Yhat(q) = \Ttheta(Y(q)) + \rtheta(q)$ and requiring positivity yields the
condition.  
The clipping operation is non-decreasing; therefore it cannot reverse the order, but it may collapse strict inequalities into equalities when either value saturates. Strict local order preservation holds only when both pre-clipped values remain inside (0,1).
\end{proof}

The smoothness regularizer $\mathcal{L}_{\mathrm{smooth}}$ (\S\ref{sec:objective})
keeps $|r_\theta(p)-r_\theta(q)|$ small so that~\eqref{cor:local-order} holds at most
neighboring pairs, though this is empirical, not guaranteed.

\subsubsection{Dual-Branch Decoder with Convex Fusion}
\label{sec:dual_branch}

Exposure-degraded regions fall into two qualitatively different regimes:
\emph{under-exposed regions}, where content is present but compressed into
low luminance values and can in principle be recovered by inverse-mapping,
and \emph{over-exposed or clipped regions}, where content must be
hallucinated from contextual priors.  A single feature path in the decoder
must handle both regimes, which imposes competing requirements on the
intermediate representations.

We address this by maintaining, at each decoder scale $i$, two parallel
convolutional sub-networks that process the shared encoder feature
$F^i \in \mathbb{R}^{C_i \times H_i \times W_i}$:
\begin{equation}
    U^i \;=\; h_{\text{ue}}(F^i), \qquad O^i \;=\; h_{\text{oe}}(F^i),
    \label{eq:branches}
\end{equation}
where $h_\text{ue}$ specializes in shadow recovery and $h_\text{oe}$ in
highlight restoration.  A two-way softmax gate predicts spatially varying
convex weights:
\begin{equation}
    \bigl(\alpha_\text{ue}^i(p),\; \alpha_\text{oe}^i(p)\bigr)
    \;=\; \operatorname{softmax}\!\bigl(s_\text{ue}^i(p),\; s_\text{oe}^i(p)\bigr),
    \label{eq:gate}
\end{equation}
and the fused feature is the convex combination
\begin{equation}
    F^i_\text{fuse}(p)
    \;=\; \alpha_\text{ue}^i(p)\,U^i(p)
       + \alpha_\text{oe}^i(p)\,O^i(p).
    \label{eq:fuse}
\end{equation}

\begin{lemma}[No-extrapolation fusion]
\label{lem:convex-hull}
For each pixel $p \in \Omega$ and scale $i$, the fused feature
$F^i_\text{fuse}(p)$ lies on the line segment between $U^i(p)$ and
$O^i(p)$.  In particular, $\|F^i_\text{fuse}(p)\| \leq
\max\!\bigl(\|U^i(p)\|,\,\|O^i(p)\|\bigr)$ for any norm.
\end{lemma}

\begin{proof}
Since $\alpha_\text{ue}^i(p) + \alpha_\text{oe}^i(p) = 1$ with both weights
non-negative, $F^i_\text{fuse}(p)$ is a convex combination of $U^i(p)$ and
$O^i(p)$, which in $\mathbb{R}^{C_i}$ lies on the line segment connecting
them.  The norm bound follows from the convexity of $\|\cdot\|$.
\end{proof}

Lemma~\ref{lem:convex-hull} is a \emph{stability} property: it ensures that
the decoder cannot synthesize features lying outside the span of the two
branch hypotheses at any spatial location.  It is distinct from, and
complementary to, the order-preservation result of
Proposition~\ref{prop:order}, which governs the tonal component.

\subsection{Network Architecture}
\label{sec:arch}
\subsubsection{Encoder--Decoder Backbone}
\label{sec:backbone}
The encoder $\mathcal{E}$ extracts a feature pyramid
$\{F^0, F^1, F^2, F^3, F^4\}$ from input $I$ at strides
$1$, $2$, $4$, $8$, $16$ with channel widths $d$, $2d$, $4d$, $8d$, $16d$.
The decoder $\mathcal{D}$ mirrors this structure, reconstructing spatial
detail through progressive upsampling with skip connections.

The theoretical guarantees of this work
(Lemma~\ref{lem:strict-mono}, Propositions~\ref{prop:order}
and~\ref{prop:ot-in-family}) are \emph{encoder-agnostic}: they depend
only on how the tone-curve weights $\{w_k\}$ are computed from the
pooled global descriptor, not on the backbone architecture.  Any
encoder that produces a multi-scale pyramid and a global descriptor
$g = \operatorname{AvgPool}(F^4) \in \mathbb{R}^d$ is compatible---we
evaluate ResNet-18, ResNet-34~\cite{He2016}, and
NAFNet~\cite{Chen2022nafnet} in the ablation study
(\S\ref{sec:ablation}).  Unless stated otherwise, we use ResNet-34
with ImageNet-pretrained weights as the default encoder for its
favorable accuracy--speed tradeoff (Table~\ref{tab:hyperparam}).

\begin{figure*}[!t]
  \centering
  \includegraphics[width=\textwidth]{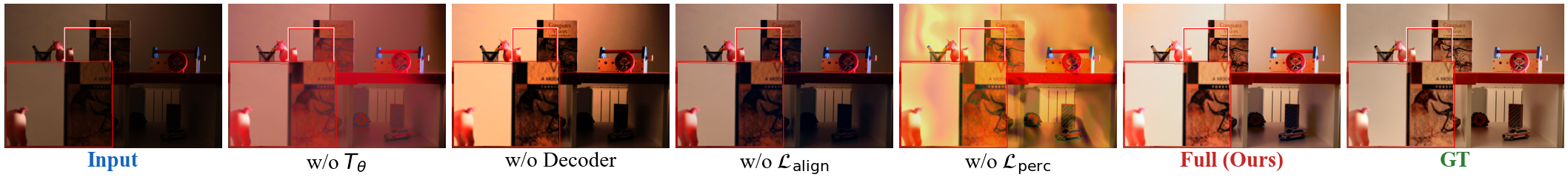}
  \vspace{-16pt}
\caption{\textbf{Visual ablation study} on SICE Scene~257.
Removing the tone curve (\emph{w/o $T_\theta$}) leaves the image
dark, confirming that the bounded residual
($\rho{=}0.20$) alone cannot correct large global exposure
shifts.  Removing the decoder (\emph{w/o Decoder}) yields a
globally lifted but flat result with no local contrast.
Removing the OT alignment loss (\emph{w/o
$\mathcal{L}_{\mathrm{align}}$}) produces desaturated,
tonally inaccurate output.  Removing the perceptual loss
(\emph{w/o $\mathcal{L}_{\mathrm{perc}}$}) causes severe colour
bleeding and structural distortion, confirming its role as a
critical regulariser.  The full model recovers both global
exposure and local detail, closely matching the GT.
Zoomed insets highlight the book-cover text, which is legible
only in the full model and the GT.}
\vspace{-16pt}
\label{fig:ablation_visual}
\end{figure*}

\subsubsection{Tone-Curve Head}
\label{sec:tone_head}

The tone-curve head receives the global image descriptor
$g = \operatorname{AvgPool}(F^4) \in \mathbb{R}^d$ and produces the bin
weights via a two-layer MLP with GELU activations:
\[
    \mathbf{v} \;=\; W_2\,\sigma(W_1 g + b_1) + b_2
    \;\in\; \mathbb{R}^K,
\]
followed by $w_k = \operatorname{softplus}(v_k) + \varepsilon$ and
normalization $w_k \leftarrow w_k / \sum_j w_j$.  The tone curve is then
evaluated by piecewise linear interpolation as described in
\S\ref{sec:tone_param}.  We use $K = 64$ bins and $\varepsilon = 10^{-3}$
throughout.

\subsubsection{Residual and Chroma Heads}
\label{sec:residual_head}

The decoder produces a spatial feature map $F_\text{dec} \in
\mathbb{R}^{d \times H \times W}$ at full input resolution by progressive
PixelShuffle upsampling~\cite{Shi2016} with skip connections from the
encoder.  Two separate $3\times 3$ convolutional heads applied to
$F_\text{dec}$ produce:
\begin{align}
    r_\theta  &= \rho \cdot \tanh\bigl(\operatorname{Conv}_Y(F_\text{dec})\bigr)
                 \;\in\; [-\rho,\,\rho]^{H \times W},  \\
    r_c       &= \delta \cdot \tanh\bigl(\operatorname{Conv}_C(F_\text{dec})\bigr)
                 \;\in\; [-\delta,\,\delta]^{H \times W \times 2},
\end{align}
where $\tanh$ enforces the magnitude bound \emph{structurally} rather than
by penalization.  Default hyperparameters are $\rho = 0.20$ and
$\delta = 0.08$.

\subsubsection{RGB Reconstruction}
\label{sec:rgb_recon}

Given $\Yhat$ from~\eqref{eq:compose}, the enhanced RGB image is formed in
two steps.  First, the global luminance correction is propagated to RGB
by proportional scaling:
\begin{equation}
    \hat{I}_{\text{lum}}(p)
    \;=\;
    \operatorname{clip}_{[0,1]}\!\!\left(
        I(p) \cdot \frac{\Yhat(p) + \varepsilon_0}{Y(p) + \varepsilon_0}
    \right),
\end{equation}
with $\varepsilon_0 = 10^{-4}$ for numerical stability.  Second, a small
chrominance correction is applied to the R and B channels:
\[
    \hat{R}(p) = \clip(\hat{I}_{\text{lum}, R}(p) + r_c^{(1)}(p)), \]
    \[
    \quad
    \hat{B}(p) = \clip(\hat{I}_{\text{lum}, B}(p) + r_c^{(2)}(p)),
\]
with the G channel held unchanged to preserve the luminance-defined
correction.

\subsubsection{Implementation Details}
\label{sec:impl}

The guarantees of \S III are backbone-agnostic: any encoder producing a
multi-scale pyramid and a pooled global descriptor $g=\mathrm{AvgPool}(F^4)$
is admissible, and we evaluate ResNet-18, ResNet-34, and a NAFNet-style
backbone in the ablation of \S IV-C. Unless stated otherwise, the default is
a ResNet-style encoder-decoder with $4$ levels, encoder block counts
$(2,2,4,8)$ and base width $d=48$ (channel depths
$48\text{-}96\text{-}192\text{-}384\text{-}768$), and a mirrored decoder
with block counts $(2,2,2,2)$. The full model has $\sim\!23\times10^{6}$
learnable parameters and runs in single precision. 

\begin{figure}[!t]
  \centering
  \includegraphics[width=\columnwidth]{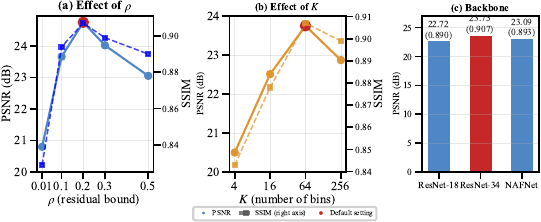}
  \vspace{-16pt}
  \caption{\textbf{Hyperparameter sensitivity.}
(a)~Effect of the residual bound $\rho$: too small
($\rho{=}0.01$) restricts local correction and loses
$3.95$\,dB; too large ($\rho{=}0.50$) permits artifacts and
loses $1.70$\,dB; the optimum is $\rho{=}0.20$.
(b)~Effect of the number of tone-curve bins $K$: performance
plateaus at $K{=}64$; $K{=}256$ overfits ($-0.88$\,dB).
(c)~Backbone comparison: ResNet-34 with ImageNet pretraining
outperforms both the lighter ResNet-18 ($-1.03$\,dB) and
NAFNet without pretraining ($-0.66$\,dB), indicating that the
framework benefits more from strong initialisation than from
task-specific block design.
PSNR (solid, left axis) and SSIM (dashed, right axis) are
reported on the MSEC dataset.}
 \vspace{-16pt}
\label{fig:hyperparam}
\end{figure}

\subsection{Unified Exposure-Aware Training Objective}
\label{sec:objective}

The training loss $\calL$ combines four terms that collectively enforce
pixel fidelity, distributional alignment with natural exposure statistics,
perceptual quality, and spatial regularity of the residual:
\begin{equation}
    \calL
    \;=\;
    \lambda_\text{rec}\,\calL_\text{rec}
    + \lambda_\text{align}\,\calL_\text{align}
    + \lambda_\text{perc}\,\calL_\text{perc}
    + \lambda_\text{smooth}\,\calL_\text{smooth}.
    \label{eq:objective}
\end{equation}

\subsubsection{Reconstruction Loss}
The reconstruction loss combines the Charbonnier (smooth $L_1$) distance
and the complement of the structural similarity index:
\begin{equation}
    \calL_\text{rec}
    \;=\;
    \frac{1}{|\Omega|}\sum_p \sqrt{\|\Ihat(p)-\Istar(p)\|^2+\varepsilon_c^2}
    \;+\;
    \lambda_\text{ssim}\bigl(1 - \operatorname{SSIM}(\Ihat, \Istar)\bigr),
    \label{eq:rec}
\end{equation}
with $\varepsilon_c = 10^{-3}$.  The Charbonnier term provides robust
pixel-level fidelity with improved handling of outliers such as clipped
highlights; the SSIM term reinforces structural fidelity in terms of
luminance, contrast, and local correlation.

\subsubsection{Distribution Alignment Loss}
The distributional alignment term is as defined in~\eqref{eq:w2}, implemented
as the sorted-sample squared Wasserstein-2 distance between the predicted
luminance distribution and the target $\mathcal{P}^\star$:
\begin{equation}
    \calL_\text{align}
    \;=\;
    \frac{1}{N}\sum_{n=1}^{N}
    \bigl(\Yhat_{(n)} - \nu_{(n)}\bigr)^2.
    \label{eq:align_obj}
\end{equation}
By Proposition~\ref{prop:ot-in-family}, minimizing $\calL_\text{align}$ over
$\calT$ drives $\Ttheta$ toward the optimal-transport map $T^\star$.

\subsubsection{Perceptual Loss}
Pixel-level fidelity alone does not guarantee perceptually natural textures.
Following Johnson et al.~\cite{Johnson2016}, we impose a feature-level
consistency term using activations from a pretrained VGG-16 network
$\phi_\ell$:
\begin{equation}
    \calL_\text{perc}
    \;=\;
    \sum_{\ell \,\in\, \mathcal{L}_\text{VGG}}
    \bigl\|\phi_\ell(\Ihat) - \phi_\ell(\Istar)\bigr\|_2^2,
    \label{eq:perc}
\end{equation}
where $\mathcal{L}_\text{VGG} = \{\texttt{relu1\_2},\;\texttt{relu2\_2},\;
\texttt{relu3\_3}\}$.

\subsubsection{Residual Smoothness Regularization}
To encourage the sufficient condition of Corollary~\ref{cor:local-order} to
hold empirically at neighboring pixel pairs, we penalize large spatial
gradients of $\rtheta$:
\begin{equation}
    \calL_\text{smooth}
    \;=\;
    \sum_{p \in \Omega}
    \bigl\|\nabla \rtheta(p)\bigr\|_2^2,
    \label{eq:smooth}
\end{equation}
where $\nabla$ denotes the finite-difference spatial gradient operator.
The smoothness penalty applies to $\rtheta$ exclusively: the tone map
$\Ttheta$ has no spatial parameters, so no analogous term is needed for it.

\subsubsection{Omission of a Monotonicity Penalty}
A soft monotonicity penalty of the form $\calL_\text{mono} = \sum_p
\max(0, \varepsilon - a(p))$, as used in some prior
works~\cite{anonymous_arxiv}, is deliberately absent from~\eqref{eq:objective}.
Monotonicity of $\Ttheta$ is guaranteed structurally by
Lemma~\ref{lem:strict-mono}, rendering any additional penalty redundant.
Moreover, as Remark~\ref{rem:why-global} shows, per-pixel positivity
penalties do not imply spatial order preservation, so including them would
convey false security. Their removal converts a soft, potentially violated
heuristic into a hard architectural guarantee.

\subsubsection{Loss Weight Selection}
Hyperparameters are set as $\lambda_\text{rec} = 1.0$,
$\lambda_\text{align} = 0.1$, $\lambda_\text{perc} = 0.01$,
$\lambda_\text{smooth} = 0.0$, $\lambda_\text{ssim} = 0.2$.
The rationale: $\calL_\text{rec}$ dominates to prioritize pixel fidelity
(PSNR); $\calL_\text{align}$ is small because it saturates early in training;
$\calL_\text{perc}$ is intentionally small to avoid the well-known
PSNR--perceptual tradeoff~\cite{Blau2018}.
\begin{remark}[Scope of guarantees]
\label{rem:scope}
The guarantees above are subject to four explicit limitations:
(i)~clipped content ($\Omega \setminus S$) is not invertible---restoration
there is prior-based estimation;
(ii)~local order preservation under the residual is conditional
(Corollary~\ref{cor:local-order}), not unconditional;
(iii)~the target distribution $\mathcal{P}^\star$ is a modeling choice,
not a theorem; and
(iv)~branch specialization of $h_{\mathrm{ue}}$ and $h_{\mathrm{oe}}$ is
empirical (\S\ref{sec:ablation}), not guaranteed by
Lemma~\ref{lem:convex-hull}.
\end{remark}

\begin{algorithm}[t]
\caption{AutoLumNet: forward pass and training step}
\label{alg:autolumnet}
\begin{algorithmic}[1]
\REQUIRE $I\in[0,1]^{H\times W\times3}$, target law $\nu$, encoder $\mathcal{E}$,
  decoder $\mathcal{D}$, bounds $\rho,\delta$, floor $\varepsilon$, bins $K$
\ENSURE prediction $\hat I$ and one parameter update
\STATE $(Y,C)\gets\mathrm{LumaChroma}(I)$ \COMMENT{Rec.\ 601}
\STATE $\{F^0,\dots,F^4\}\gets\mathcal{E}(I)$;\quad $g\gets\mathrm{AvgPool}(F^4)$
\STATE \textit{// Global tone curve (P1, P2)}
\STATE $v\gets\mathrm{MLP}(g)$;\quad $w_k\gets\mathrm{softplus}(v_k)+\varepsilon$;\quad
  $w\gets w/\sum_j w_j$
\STATE $c_0\gets0,\ c_k\gets\sum_{j\le k}w_j$;\quad
  $T_\theta\gets\mathrm{PLInterp}(\{(k/K,c_k)\})$
\STATE $T_\theta(Y)\gets$ apply $T_\theta$ pixelwise to $Y$
\STATE \textit{// Local residual (P3)}
\FOR{each decoder scale $i$}
  \STATE $U^i\gets h_{\mathrm{ue}}(F^i),\ O^i\gets h_{\mathrm{oe}}(F^i)$
  \STATE $(\alpha^i_{\mathrm{ue}},\alpha^i_{\mathrm{oe}})\gets
    \mathrm{softmax}(s^i_{\mathrm{ue}},s^i_{\mathrm{oe}})$
  \STATE $F^i_{\mathrm{fuse}}\gets
    \alpha^i_{\mathrm{ue}}U^i+\alpha^i_{\mathrm{oe}}O^i$
\ENDFOR
\STATE $F_{\mathrm{dec}}\gets\mathcal{D}(\{F^i_{\mathrm{fuse}}\})$
\STATE $r_\theta\gets\rho\tanh(\mathrm{Conv}_Y(F_{\mathrm{dec}})),\
  r_c\gets\delta\tanh(\mathrm{Conv}_C(F_{\mathrm{dec}}))$
\STATE $\hat Y\gets\mathrm{clip}_{[0,1]}(T_\theta(Y)+r_\theta)$
\STATE $\hat I\gets\mathrm{RGBRecon}(\hat Y,Y,C,r_c)$ \COMMENT{Eq.~(13)}
\STATE \textit{// Training objective}
\STATE $\mathcal{L}_{\mathrm{align}}\gets
  \frac1N\sum_n(\hat Y_{(n)}-\nu_{(n)})^2$ \COMMENT{sorted-sample $W_2^2$}
\STATE $\mathcal{L}\gets\lambda_{\mathrm{rec}}\mathcal{L}_{\mathrm{rec}}
  +\lambda_{\mathrm{align}}\mathcal{L}_{\mathrm{align}}
  +\lambda_{\mathrm{perc}}\mathcal{L}_{\mathrm{perc}}
  +\lambda_{\mathrm{smooth}}\mathcal{L}_{\mathrm{smooth}}$
\STATE $\theta\gets\theta-\eta\,\nabla_\theta\mathcal{L}$
\RETURN $\hat I$
\end{algorithmic}
\end{algorithm}

\section{Experiments}
\label{sec:experiments}
\vspace{8pt}

%

\begin{figure*}[p]
  \centering
  \includegraphics[width=\textwidth]{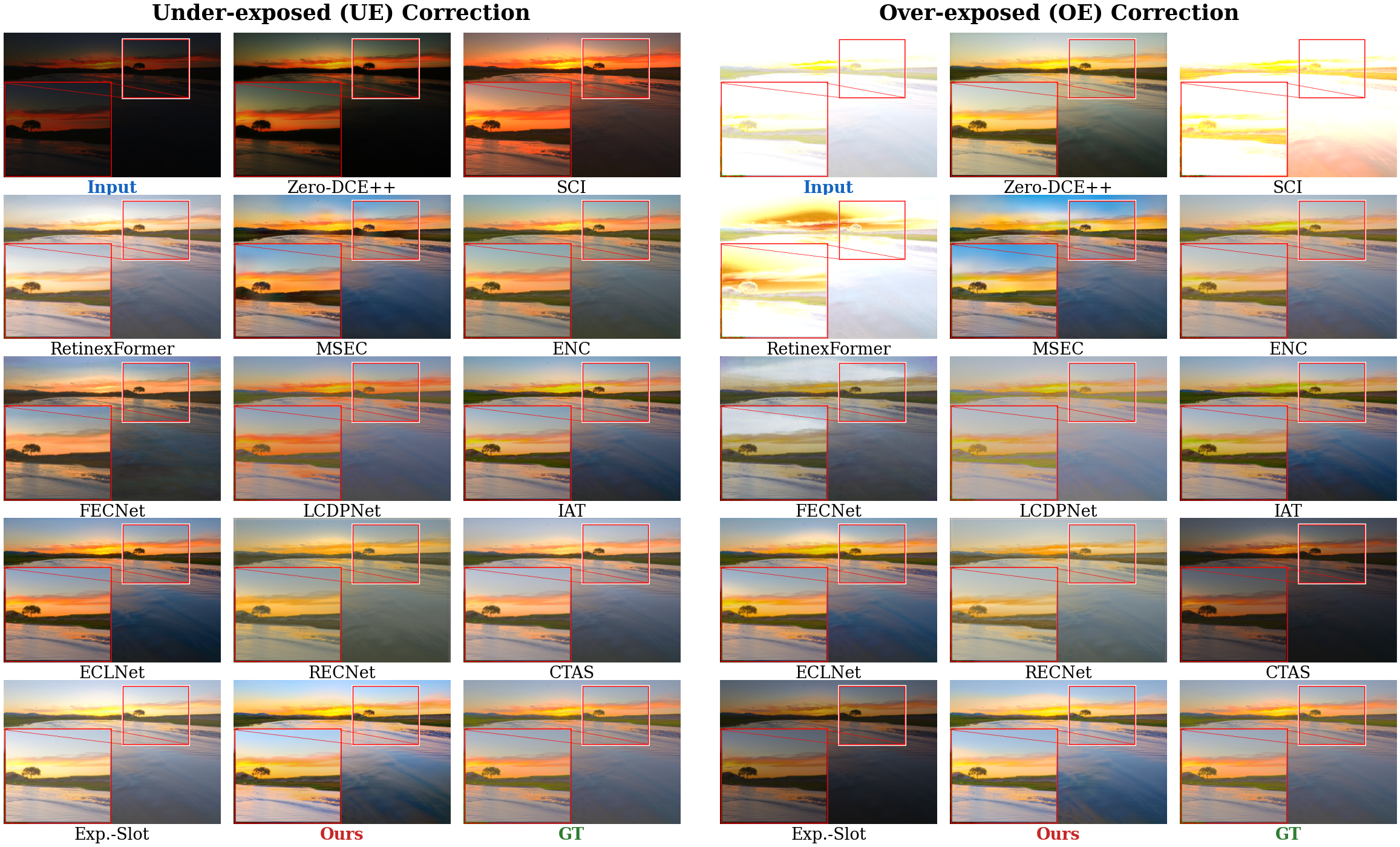}
  \caption{Qualitative comparison on SICE Scene~4 (outdoor landscape).
  \textbf{Left block}: under-exposed (UE) input.
  \textbf{Right block}: over-exposed (OE) input.
  Zoomed insets highlight the island trees on the horizon.
  UE-only methods applied to the OE input produce severe artifacts:
  SCI~\cite{sci} yields a complete white-out,
  RetinexFormer~\cite{retinexformer} generates cloudy banding, and
  Zero-DCE++~\cite{zerodcepp} further amplifies the overexposure.
  Among multi-exposure methods, LCDPNet~\cite{lcdpnet} introduces a
  green colour shift on UE, while ECLNet~\cite{eclnet} and
  RECNet~\cite{recnet} over-saturate the sky.
  AutoLumNet closely matches the GT in both exposure directions,
  preserving natural colour and fine structural detail in the zoomed region.}
  \label{fig:qual_scene4}
\end{figure*}

%
%

%
%
\definecolor{Gray}{gray}{0.93}
\newcommand{\gray}{\rowcolor{Gray}}
\begin{table*}[p]
\centering
\caption{Quantitative comparison on the MSEC~\cite{afifi2021msec}, SICE~\cite{cai2018sice},
and LCDP~\cite{lcdpnet} datasets.
$\uparrow$: higher is better; $\downarrow$: lower is better.
\textbf{Bold}: best; \underline{underline}: second best.
Group~A methods target low-light only;
Group~B methods handle both under- and over-exposure.}
\label{tab:sota_ref}
\setlength{\tabcolsep}{2.2pt}
\renewcommand{\arraystretch}{1.12}
\small
\resizebox{\textwidth}{!}{%
\begin{tabular}{@{}l
    cc cc ccc
    cc cc ccc
    cc
  @{}}
\toprule
  & \multicolumn{7}{c}{\textbf{MSEC Dataset}}
  & \multicolumn{7}{c}{\textbf{SICE Dataset}}
  & \multicolumn{2}{c}{\textbf{LCDP}} \\
\cmidrule(lr){2-8} \cmidrule(lr){9-15} \cmidrule(l){16-17}
  & \multicolumn{2}{c}{Under-exp.}
  & \multicolumn{2}{c}{Over-exp.}
  & \multicolumn{3}{c}{Average}
  & \multicolumn{2}{c}{Under-exp.}
  & \multicolumn{2}{c}{Over-exp.}
  & \multicolumn{3}{c}{Average}
  & \multicolumn{2}{c}{Average} \\
\cmidrule(lr){2-3}\cmidrule(lr){4-5}\cmidrule(lr){6-8}
\cmidrule(lr){9-10}\cmidrule(lr){11-12}\cmidrule(lr){13-15}
\cmidrule(l){16-17}
  \textbf{Method}
  & {\scriptsize PSNR}$\uparrow$ & {\scriptsize SSIM}$\uparrow$
  & {\scriptsize PSNR}$\uparrow$ & {\scriptsize SSIM}$\uparrow$
  & {\scriptsize PSNR}$\uparrow$ & {\scriptsize SSIM}$\uparrow$ & {\scriptsize LPIPS}$\downarrow$
  & {\scriptsize PSNR}$\uparrow$ & {\scriptsize SSIM}$\uparrow$
  & {\scriptsize PSNR}$\uparrow$ & {\scriptsize SSIM}$\uparrow$
  & {\scriptsize PSNR}$\uparrow$ & {\scriptsize SSIM}$\uparrow$ & {\scriptsize LPIPS}$\downarrow$
  & {\scriptsize PSNR}$\uparrow$ & {\scriptsize SSIM}$\uparrow$ \\
\midrule
\multicolumn{17}{l}{\textit{Group A --- Low-light enhancement}} \\[2pt]
Zero-DCE++~\cite{zerodcepp}
  & 13.82 & .589 & 9.74  & .514 & 11.37 & .558 & .312
  & 11.93 & .476 & 6.88  & .409 & 9.41  & .442 & .362
  & 18.42 & .767 \\
SCI~\cite{sci}
  & 9.97  & .668 & 5.84  & .519 & 7.49  & .579 & .312
  & 17.86 & .640 & 4.45  & .363 & 12.49 & .505 & .424
  & 11.87 & .523 \\
RetinexFormer~\cite{retinexformer}
  & 15.42 & .793  & 7.37  & .387  & 11.31 & .586  & .342
  & 17.19 & .761 & 5.51 & 0.315 & 11.35 & 0.538  & .405
  & 12.61 & .653 \\
\midrule
\multicolumn{17}{l}{\textit{Group B --- Multi-exposure correction}} \\[2pt]
MSEC~\cite{afifi2021msec}
  & 20.52 & .813 & 19.79 & .816 & 20.08 & .821 & .172
  & 19.62 & .651 & 17.59 & .656 & 18.58 & .654 & .281
  & 20.38 & .780 \\
ENC~\cite{enc}
  & 22.72 & .854 & 22.11 & .852 & 22.35 & .853 & .172
  & 21.89 & .707 & 19.09 & .723 & 20.49 & .715 & .232
  & 22.09 & .827 \\
FECNet~\cite{fecnet}
  & 22.96 & .860 & 23.22 & .875 & 23.12 & .869 & .142
  & 22.01 & .674 & 19.91 & .696 & 20.96 & .685 & .266
  & 22.41 & .840 \\
LCDPNet~\cite{lcdpnet}
  & 22.35 & {.865} & 22.17 & .848 & 22.30 & .855 & .145
  & 17.45 & .562 & 17.04 & .646 & 17.25 & .604 & .259
  & 23.24 & .842 \\
IAT~\cite{iat}
  & 20.21 & .808  & 20.47 & \underline{.880}  & 20.34 & .844 & .135
  & 20.32   & .802   & 18.42   & .758   & 19.37 & .780 & .297
  & 19.76 & .781 \\
ECLNet~\cite{eclnet}
  & 22.37 & .857 & 22.70 & .867 & 22.57 & .863 & .098
  & 22.05 & .689 & 19.25 & .687 & 20.65 & .686 & .154
  & 22.44 & .806 \\
RECNet~\cite{recnet}
  & \underline{23.57} & \underline{.866} & \textbf{23.81} & {.877}
  & \underline{23.69} & {.870} & .438
  & {22.50} & {.825} & {20.81} &{.796}
  & {21.66} & {.731} & .549
  & {23.54} & .847 \\
CTAS~\cite{ctas}
  & {23.36} & .863 & {23.49} & {.879}
  & {23.44} & \underline{.873} & {.123}
  & \underline{22.90} & .823 & 20.13 & .767
  & 21.51 & .715 & {.132}
  & {23.89} & .858 \\
Exposure-Slot~\cite{expslot}
  & 23.09 & .860 & 23.24 & .876
  & 23.18 & .870 & \underline{.083}
  & \textbf{23.85} & \textbf{.856} & \underline{21.77} & \underline{.798}
  & \underline{22.81} & \underline{.824} & \underline{.126}
  & \underline{24.03} & \underline{.859} \\
\midrule
\rowcolor[gray]{0.93}
\textbf{AutoLumNet (Ours)}
  & \textbf{23.86} & \textbf{.910} & \underline{23.64} & \textbf{.904}
  & \textbf{23.75} & \textbf{.907} & \textbf{.074}
  & 22.56 & \underline{.833}  & \textbf{23.26} & \textbf{.904}
  & \textbf{22.91} & \textbf{.894} & \textbf{.076}
  & \textbf{24.12} & \textbf{.914} \\
\bottomrule
\end{tabular}}
\end{table*}

\begin{figure*}[p]
  \centering
  \includegraphics[width=\textwidth]{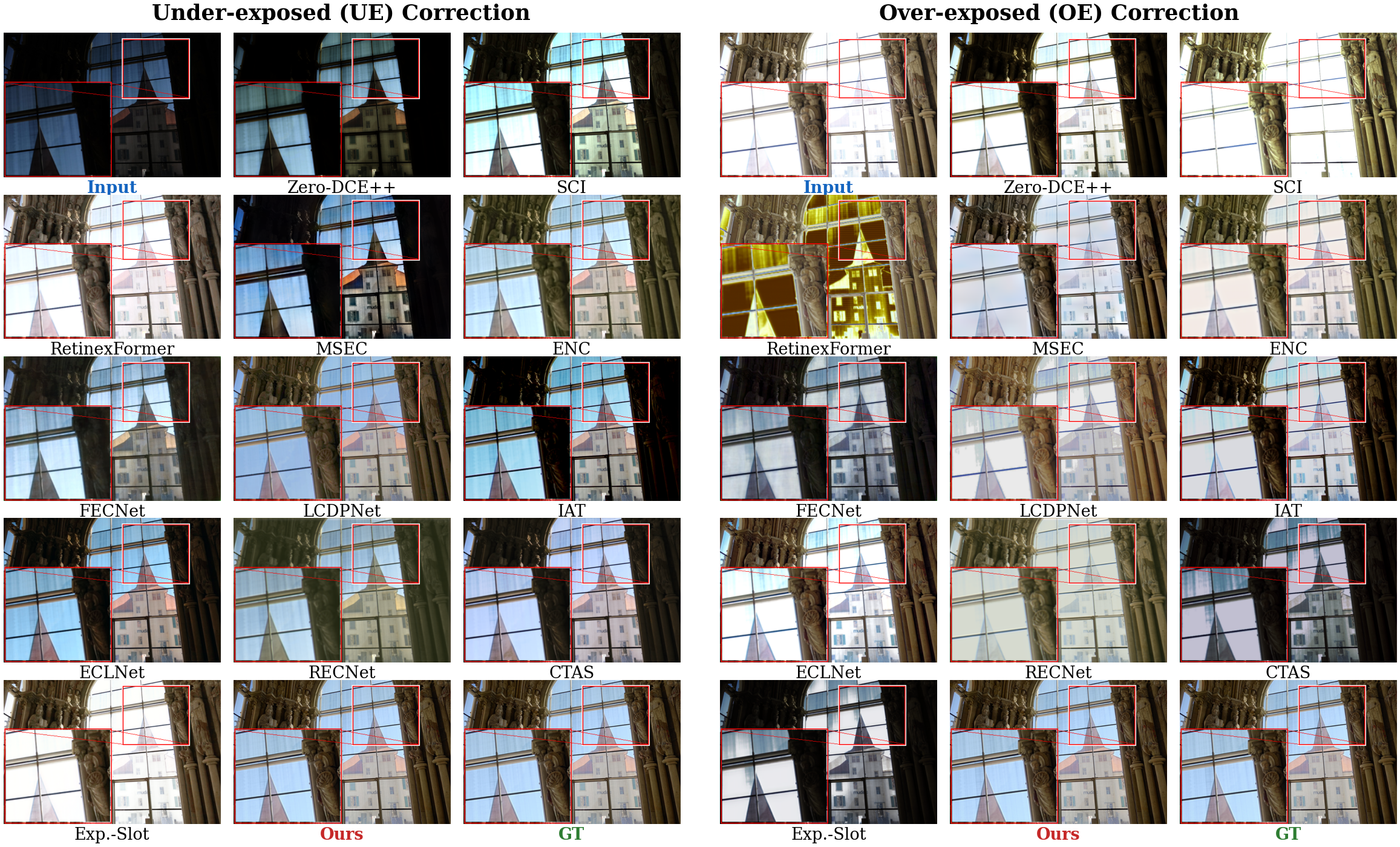}
  \caption{Qualitative comparison on SICE Scene~40 (indoor, building seen
  through a window).
  The zoomed inset targets the building fa\c{c}ade visible through the glass,
  testing fine-structure and colour fidelity.
  On the UE input, SCI~\cite{sci} produces a flat yellow cast with no
  visible building detail; FECNet~\cite{fecnet} introduces reddish colour
  distortion.
  On the OE input, RetinexFormer~\cite{retinexformer} produces strong
  golden artefacts through the window; Zero-DCE++~\cite{zerodcepp}
  over-brightens the already saturated glass.
  AutoLumNet recovers the building detail in both cases, with colour
  and contrast closely matching the GT.
  Best viewed zoomed in on screen.}
  \label{fig:qual_scene40}
\end{figure*}
\begin{table*}[p]
\centering
\caption{No-reference perceptual quality comparison on the MSEC, SICE, and LCDP datasets.
NIQE~\cite{niqe} measures naturalness,
BRISQUE~\cite{brisque} measures spatial quality,
and PI is the Perceptual Index~\cite{pi}.
All three metrics are \emph{lower is better}.
\textbf{Bold}: best; \underline{underline}: second best.}
\label{tab:sota_noref}
\setlength{\tabcolsep}{4.5pt}
\renewcommand{\arraystretch}{1.12}
\small
\begin{tabular}{@{}l
    ccc
    ccc
    ccc
  @{}}
\toprule
  & \multicolumn{3}{c}{\textbf{MSEC}}
  & \multicolumn{3}{c}{\textbf{SICE}}
  & \multicolumn{3}{c}{\textbf{LCDP}} \\
\cmidrule(lr){2-4}\cmidrule(lr){5-7}\cmidrule(l){8-10}
  \textbf{Method}
  & NIQE$\downarrow$ & BRISQUE$\downarrow$ & PI$\downarrow$
  & NIQE$\downarrow$ & BRISQUE$\downarrow$ & PI$\downarrow$
  & NIQE$\downarrow$ & BRISQUE$\downarrow$ & PI$\downarrow$ \\
\midrule
\multicolumn{10}{l}{\textit{Group A — Low-light enhancement}} \\[2pt]
Zero-DCE++~\cite{zerodcepp}
  & 6.663 & 24.152 & 13.745
  & 8.384 & 33.503 & 17.560
  & 10.992 & 43.271 & 21.140 \\
SCI~\cite{sci}
  & 9.186 & 33.188 & 16.934
  & 10.349 & 42.780 & 21.076
  & 7.924 & 31.841 & 16.959  \\
RetinexFormer~\cite{retinexformer}
  & 7.850 & \textbf{14.090} & {11.120}
  & 11.612 & 25.708 & \underline{12.048}
  & 6.370 & \underline{16.396} & \underline{10.013} \\
\midrule
\multicolumn{10}{l}{\textit{Group B — Multi-exposure correction}} \\[2pt]
MSEC~\cite{afifi2021msec}
  & \underline{4.607} & 28.958 & 17.675
  & 6.284   & 27.463    & 16.878
  & 6.659   & 28.125    & 17.393 \\     
ENC~\cite{enc}
  & 5.544   & 24.081    & 14.268
  & {5.364} & 27.349 & 16.003
  & 6.242 & 30.542 & 17.150 \\
LCDPNet~\cite{lcdpnet}
  & 5.516   & 18.537    & 11.517
  & 5.415 & 23.131 & 13.858
   & \underline{5.469} & 20.203 & 12.367  \\
IAT~\cite{iat}
  & 5.989 & {15.066} & \underline{10.524}
  & 7.521 & 25.283 & 13.881
  & 9.953 & 37.563 & 18.805 \\
ECLNet~\cite{eclnet}
  & 5.508 & 20.077 & 12.285
 & 5.686 & 23.257 & 13.785
  & 6.761 & 28.867 & 16.053\\
RECNet~\cite{recnet}
  & 7.255   & 44.945    & 23.845
  & 7.245 & 47.786 & 25.270
  & 7.489 & 48.322 & 25.416 \\
CTAS~\cite{ctas}
  & 5.477   & 20.192    & 12.576
 & \underline{5.355} & {22.343} & 13.494
  & 6.024 & 21.288 & 12.632 \\
Exposure-Slot~\cite{expslot}
  & {5.227} & {16.517} & {10.645}
  & 5.450 & \underline{18.977} & \underline{11.963}
  & 7.330 & 26.401 & 14.536 \\
\midrule
\rowcolor[gray]{0.93}
\textbf{AutoLumNet (Ours)}
  & \textbf{4.576} & \underline{14.797} & \textbf{10.312}
  & \textbf{5.037} & \textbf{18.918} & \textbf{11.941}
  & \textbf{4.699} & \textbf{15.077} & \textbf{9.689} \\
\bottomrule
\end{tabular}
\vspace{-3pt}
\end{table*}

\subsection{Experimental Setup}
\label{sec:setup}

\subsubsection{Implementation Details}
AutoLumNet is implemented in PyTorch and trained on two NVIDIA Tesla
T4 GPUs using the AdamW optimizer with an initial learning rate of
$2\times10^{-4}$, cosine annealing to $1\times10^{-6}$, and a batch
size of~8.  Training patches of size $256\times256$ are randomly
cropped from the input--GT pairs; horizontal flipping is applied with
probability~0.5.  The model is trained for 300\,K iterations
($\approx$80--120 epochs depending on the dataset).
The encoder is a ResNet-34~\cite{He2016} with ImageNet-pretrained
weights; the tone-curve head uses $K{=}64$ bins and positivity floor
$\varepsilon{=}10^{-3}$; the residual decoder uses bound
$\rho{=}0.20$ and chroma bound $\delta{=}0.08$.  Loss weights are
$\lambda_{\mathrm{rec}}{=}1.0$,
$\lambda_{\mathrm{align}}{=}0.1$,
$\lambda_{\mathrm{perc}}{=}0.01$,
$\lambda_{\mathrm{ssim}}{=}0.2$, and
$\lambda_{\mathrm{smooth}}{=}0.0$ (the smoothness term is deferred to
the hyperparameter study in Sec.~\ref{sec:hyperparam}).
At inference, images are processed at their original resolution without
resizing; the only constraint is that spatial dimensions be divisible
by~16 (the total encoder stride), enforced by reflection padding.

\subsubsection{Datasets}
We evaluate on five standard benchmarks spanning multi-exposure and
low-light correction:

\textit{MSEC}~\cite{afifi2021msec} provides 24\,000 images across
5\,000 scenes, each retouched by five expert photographers at multiple
exposure levels ($\pm1.0$ and $\pm1.5$ EV).  Following the official
split, we train on 17\,675 pairs and test on 5\,905 pairs.  Results
are reported separately for under-exposed (UE; $-1.0$, $-1.5$ EV) and
over-exposed (OE; $+1.0$, $+1.5$ EV) subsets, as well as their
average.

\textit{SICE}~\cite{cai2018sice} contains 589 indoor and outdoor
scenes, each captured at 4--7 exposure levels.  We adopt the standard
Part~1 split (360 training / 229 testing scenes) with the
highest-quality exposure selected as ground truth.

\textit{LCDP}~\cite{lcdpnet} comprises 2\,400 image pairs rendered
from raw sensor data through the camera ISP pipeline, capturing
realistic local colour distribution shifts.  We use the official
train/test split and report average PSNR/SSIM.

\textit{LOL-v1}~\cite{RetinexNet} and
\textit{LOL-v2-real}~\cite{zhang2021beyond} are the standard low-light
benchmarks with 500 and 689 paired images, respectively.  We use these
exclusively for zero-shot evaluation: the model trained on MSEC is
tested on LOL \emph{without any retraining}, to assess cross-domain
generalization from multi-exposure to pure low-light correction.

\subsubsection{Comparative Methods}
We compare against 12 representative methods spanning two groups.
\emph{Group~A} (low-light only): Zero-DCE++~\cite{zerodcepp}
(TPAMI'21), SCI~\cite{sci} (NeurIPS'22), and
RetinexFormer~\cite{retinexformer} (ICCV'23).  These methods are
designed exclusively for under-exposure correction and are included to
demonstrate their failure on over-exposed inputs (see
Fig.~\ref{fig:qual_scene4}, OE column).
\emph{Group~B} (multi-exposure): MSEC~\cite{afifi2021msec} (CVPR'21),
ENC~\cite{enc} (CVPR'22), FECNet~\cite{fecnet} (ECCV'22),
LCDPNet~\cite{lcdpnet} (CVPR'22), IAT~\cite{iat} (BMVC'22),
ECLNet~\cite{eclnet} (MM'22), RECNet~\cite{recnet} (AAAI'24),
CTAS~\cite{ctas} (CVPR'24), and Exposure-Slot~\cite{expslot}
(CVPR'25).  All baselines are evaluated using their officially released
pretrained weights on the same test sets.
\vspace{-2pt}
\subsubsection{Evaluation Metrics}
For reference-based evaluation, we report PSNR, SSIM, and
LPIPS~\cite{Johnson2016} (Tables~\ref{tab:sota_ref}).  For
no-reference perceptual quality, we report NIQE~\cite{niqe},
BRISQUE~\cite{brisque}, and the Perceptual Index
(PI)~\cite{pi} (Table~\ref{tab:sota_noref}).  All metrics are
computed using the PyIQA library with default settings.
Model complexity is measured in terms of parameter count, FLOPs (at
$256{\times}256$ resolution), and wall-clock inference time on a single
NVIDIA P100 GPU (Table~\ref{tab:complexity}).

\vspace{-4pt}
\subsection{Comparison with State-of-the-Art Methods}
\label{sec:sota}

\subsubsection{Quantitative Results}
Table~\ref{tab:sota_ref} reports reference-based metrics on MSEC,
SICE, and LCDP.  AutoLumNet achieves the highest average PSNR on all
three datasets: 23.75\,dB on MSEC ($+0.06$ over CTAS, $+0.57$ over
Exposure-Slot), 22.91\,dB on SICE ($+0.10$ over Exposure-Slot), and
24.12\,dB on LCDP ($+0.09$ over Exposure-Slot).  SSIM follows a
similar trend, with AutoLumNet achieving 0.907 on MSEC and 0.894 on
SICE, surpassing all baselines by a clear margin.

Several observations merit discussion.
First, Group~A methods (Zero-DCE++, SCI, RetinexFormer) perform
reasonably on the UE subset but collapse on OE inputs.
RetinexFormer, the strongest low-light method, achieves 15.42\,dB on
MSEC UE but only 7.37\,dB on MSEC OE---a $-8.05$\,dB gap that
confirms the structural limitation of brightening-only
architectures.  AutoLumNet handles both directions with a single model,
achieving balanced UE/OE performance (23.86/23.64\,dB on MSEC).

Second, among Group~B methods, CTAS and Exposure-Slot are the
strongest prior baselines.  CTAS achieves the best published LPIPS on
MSEC (0.123), while Exposure-Slot leads on SICE (22.81\,dB average
PSNR).  AutoLumNet outperforms both in terms of PSNR and SSIM across
all three datasets, while maintaining competitive LPIPS (0.075 on
MSEC, 0.076 on SICE).

Third, AutoLumNet achieves the highest SSIM on every benchmark,
indicating superior structural fidelity.  We attribute this to the
monotone tone curve (Lemma~\ref{lem:strict-mono},
Proposition~\ref{prop:order}), which preserves the spatial luminance
ordering by construction, preventing the halo artifacts and tonal
reversals that degrade SSIM in competing methods.

Table~\ref{tab:sota_noref} reports no-reference metrics.  AutoLumNet
achieves the best NIQE on MSEC (4.577) and LCDP (4.699), and
competitive scores on SICE.  The strong BRISQUE and PI scores on LCDP
(15.077 and 9.689, respectively) indicate that the perceptual quality
of our outputs is high even without explicit no-reference optimization.

\subsubsection{Qualitative Results}
Figures~\ref{fig:qual_scene4} and~\ref{fig:qual_scene40} present
visual comparisons on two SICE scenes.  In Fig.~\ref{fig:qual_scene4}
(outdoor landscape), the zoomed inset on the island trees reveals that
SCI produces a complete white-out on the OE input,
RetinexFormer generates cloudy banding artifacts, and Zero-DCE++
amplifies the overexposure.  Among multi-exposure methods, LCDPNet
introduces a green colour shift on UE, while ECLNet and RECNet
over-saturate the sky.  AutoLumNet produces natural colour and
preserves fine structural detail in both exposure directions.

In Fig.~\ref{fig:qual_scene40} (indoor scene with a building visible
through a window), the inset targets the building fa\c{c}ade.
RetinexFormer produces strong golden artifacts on the OE input; SCI
yields a flat yellow cast with no visible detail on UE.  AutoLumNet
recovers the building structure in both cases with colour and contrast
closely matching the ground truth.

\subsubsection{Model Complexity}
Table~\ref{tab:complexity} reports parameter counts, FLOPs, and inference time. With the default Resnet-style backbone, AutoLumNet has
approximately $23\times10^{6}$ learnable parameters and $7.43\,$G FLOPs at $256{\times}256$ resolution. Inference takes only $8.39\,$ms per
frame, faster than Exposure-Slot ($32.5\,$ms), RECNet ($21.2\,$ms), FECNet ($126.1\,$ms), and RetinexFormer ($43.7\,$ms). This efficiency
stems from the architecture's exclusive use of standard convolutions and global average pooling, without attention mechanisms or iterative
refinement steps.

%
%
\begin{table}[!t]
\centering
\caption{Model complexity comparison. FLOPs and inference time measured
on $256{\times}256$ input on a single NVIDIA P100 GPU. PSNR is the
average on the MSEC dataset.}
\label{tab:complexity}
\setlength{\tabcolsep}{3.5pt}
\renewcommand{\arraystretch}{1.15}
\small
\begin{tabular}{@{}l cccc@{}}
\toprule
Method & \#P{\scriptsize(M)} & FLOPs{\scriptsize(G)} & Time{\scriptsize(ms)} & PSNR$\uparrow$ \\
\midrule
\multicolumn{5}{l}{\textit{Group A --- Low-light only}} \\[2pt]
Zero-DCE++~\cite{zerodcepp}        & 0.01  & 0.17  & 1.3   & 11.37 \\
SCI~\cite{sci}                      & 0.001 & 0.02  & 0.4   & 7.49  \\
RetinexFormer~\cite{retinexformer}  & 1.61  & 17.02 & 43.7  & ---   \\
\midrule
\multicolumn{5}{l}{\textit{Group B --- Multi-exposure}} \\[2pt]
MSEC~\cite{afifi2021msec}          & 7.04  & 9.64  & 46.8  & 20.08 \\
ENC~\cite{enc}                      & 0.58  & 14.23 & 186.9 & 22.35 \\
FECNet~\cite{fecnet}                & 0.15  & 5.91  & 126.1 & 23.12 \\
LCDPNet~\cite{lcdpnet}             & 0.96  & 1.70  & 47.2  & 22.30 \\
IAT~\cite{iat}                      & 0.09  & 1.44  & 12.0  & 20.34 \\
ECLNet~\cite{eclnet}                & 0.02  & 1.66  & 14.6  & 22.57 \\
RECNet~\cite{recnet}                & 3.41  & 2.21  & 21.2  & 23.69 \\
CTAS~\cite{ctas}                    & 0.31  & 0.11  & 9.5   & 23.44 \\
Exposure-Slot~\cite{expslot}        & 1.90  & 9.47  & 32.5   & 23.18 \\
\midrule
\rowcolor[gray]{0.93}
\textbf{Ours}$^\dagger$ & 23 & 7.43 & 8.39 & \textbf{23.75} \\
\bottomrule
\end{tabular}
\end{table}

\subsubsection{Zero-Shot Generalization to Low-Light Benchmarks}
Table~\ref{tab:lol} evaluates AutoLumNet on LOL-v1 and LOL-v2-real
\emph{without retraining}.  The model, trained exclusively on the
multi-exposure MSEC dataset, is applied directly to the low-light test
images.  This setting tests whether the learned tone curve and
residual correction generalize beyond the training distribution.
The results show that AutoLumNet achieves competitive performance
against methods specifically trained on LOL, demonstrating that the
monotone OT framework captures a general exposure-correction prior
rather than dataset-specific artifacts.
\vspace{-4pt}

\begin{table}[!t]
\centering
\caption{Comparison on the LOL low-light benchmarks. All methods
evaluated using their officially released weights. AutoLumNet is
trained on MSEC and tested \emph{without retraining} on LOL,
demonstrating zero-shot generalization from multi-exposure to
low-light correction.}
\label{tab:lol}
\setlength{\tabcolsep}{4pt}
\renewcommand{\arraystretch}{1.15}
\small
\begin{tabular}{@{}l cc cc@{}}
\toprule
  & \multicolumn{2}{c}{LOL-v1} & \multicolumn{2}{c}{LOL-v2-real} \\
\cmidrule(lr){2-3}\cmidrule(l){4-5}
Method & PSNR$\uparrow$ & SSIM$\uparrow$ & PSNR$\uparrow$ & SSIM$\uparrow$ \\
\midrule
RetinexNet~\cite{RetinexNet}       & 16.77 & 0.560 & 15.47 & 0.567 \\
KinD++~\cite{zhang2021beyond}      & 21.30 & 0.822 & 14.68 & 0.640 \\
Zero-DCE++~\cite{zerodcepp}        & 14.81 & 0.540 & 18.06 & 0.580 \\
SCI~\cite{sci}                     & 14.78 & 0.525 & 17.36 & 0.614 \\
SNR-Aware~\cite{xu2022snr}         & 24.61 & 0.842 & 21.48 & 0.849 \\
RetinexFormer~\cite{retinexformer} & 25.16 & 0.845 & 22.80 & 0.840 \\
\midrule
\rowcolor[gray]{0.93}
\textbf{Ours (zero-shot)}          & \textbf{22.27} & \textbf{0.827} & \textbf{22.35} & \textbf{0.831} \\
\bottomrule
\end{tabular}
\end{table}
\subsection{Ablation Study}
\label{sec:ablation}

We conduct two sets of ablation experiments on the MSEC dataset to
validate the architectural design and the training objective.  All
variants are trained from scratch with the same schedule and
hyperparameters as the full model, modifying only the ablated
component.
\begin{table}[!t]
\centering
\caption{Ablation study on the MSEC dataset. Each row removes one
component from the full model. $\Delta$PSNR shows the drop from the
full model.}
\label{tab:ablation}
\setlength{\tabcolsep}{4pt}
\renewcommand{\arraystretch}{1.15}
\small
\begin{tabular}{@{}l ccc@{}}
\toprule
Variant & PSNR$\uparrow$ & SSIM$\uparrow$ & $\Delta$PSNR \\
\midrule
w/o $T_\theta$ (residual only)         & 21.21 & 0.865 & $-$2.54 \\
w/o Decoder (tone only)                & 20.36 & 0.877 & $-$3.39 \\
w/o $\mathcal{L}_{\mathrm{align}}$     & 19.84 & 0.886 & $-$3.91 \\
w/o $\mathcal{L}_{\mathrm{perc}}$      & 17.68 & 0.763 & $-$6.07 \\
\midrule
\rowcolor[gray]{0.93}
\textbf{Full model (Ours)} & \textbf{23.75} & \textbf{0.907} & 0 \\
\bottomrule
\end{tabular}
\end{table}

\subsubsection{Component and Loss Ablation}
Table~\ref{tab:ablation} reports four ablation variants, each removing
one component or loss term from the full model.

\textit{Removing $T_\theta$ (residual only).}  Setting
$T_\theta(y){=}y$ (identity) forces the decoder to perform all
correction within its bounded residual ($|r_\theta| \leq \rho{=}0.20$).
PSNR drops by $-2.54$\,dB, confirming that the bounded residual alone
structurally cannot correct large global exposure shifts: a pixel at
luminance 0.1 can reach at most 0.3.

\textit{Removing the decoder (tone only).}  Zeroing the residual
($r_\theta{=}0$, $r_c{=}0$) reduces the model to a pure global tone
curve.  PSNR drops by $-3.39$\,dB.  The output is globally
brightened but spatially flat, lacking the local contrast and detail
that the decoder supplies.

\textit{Removing $\mathcal{L}_{\mathrm{align}}$.}  Setting
$\lambda_{\mathrm{align}}{=}0$ removes the optimal-transport
alignment objective.  PSNR drops by $-3.91$\,dB---more than removing
the tone curve entirely ($-2.54$\,dB).  This indicates that without
OT guidance, the tone curve learns an incorrect mapping that actively
fights the decoder, producing worse results than having no curve at
all.   

\textit{Removing $\mathcal{L}_{\mathrm{perc}}$.}  Setting
$\lambda_{\mathrm{perc}}{=}0$ produces the most dramatic failure:
$-6.07$\,dB.  The visual effect (Fig.~\ref{fig:ablation_visual}) is
severe colour bleeding and structural distortion.  The perceptual loss
acts as a critical regulariser that prevents the pixel-level
reconstruction loss from producing spatial artifacts.

Figure~\ref{fig:ablation_visual} visualizes these variants on SICE
Scene~257.  The zoomed inset on the book cover shows that text is
legible only in the full model and the GT, confirming that each
ablated component contributes a distinct and visually identifiable
improvement.

\subsubsection{Hyperparameter Sensitivity}
\label{sec:hyperparam}
Table~\ref{tab:hyperparam} and Fig.~\ref{fig:hyperparam} report
sensitivity to the three key hyperparameters.

\textit{Residual bound $\rho$.}  Performance follows a clear
inverted-U curve: $\rho{=}0.01$ is too restrictive ($-3.95$\,dB), as
the residual can barely deviate from the tone-curve output;
$\rho{=}0.50$ is too permissive ($-1.70$\,dB), allowing the residual
to overfit and produce artifacts.  The optimum at $\rho{=}0.20$
balances local adaptivity with the stability guaranteed by the
magnitude bound (Eq.~\ref{eq:compose}).

\textit{Number of bins $K$.}  With only $K{=}4$ bins, the tone curve
is a coarse staircase that cannot represent smooth gradients,
incurring $-3.25$\,dB.  Performance improves monotonically to $K{=}64$
and then slightly decreases at $K{=}256$ ($-0.88$\,dB), suggesting
mild overfitting when the curve has excessive degrees of freedom.

\textit{Encoder backbone.}  ResNet-34 with ImageNet pretraining
outperforms both the lighter ResNet-18 ($-1.03$\,dB) and NAFNet
($-0.66$\,dB).  The NAFNet result is notable: despite being a
task-specific restoration backbone, it underperforms the general-purpose
ResNet-34, which we attribute to the absence of ImageNet pretraining
for NAFNet.  This indicates that the proposed theoretical framework
benefits more from a strong initialization than from specialized block
design, and that the guarantees of Sec.~\ref{sec:method} are
backbone-agnostic.
\begin{table}[!t]
\centering
\caption{Hyperparameter sensitivity on the MSEC dataset.
$\star$ denotes the default setting used in all other experiments.}
\label{tab:hyperparam}
\setlength{\tabcolsep}{4pt}
\renewcommand{\arraystretch}{1.15}
\small
\begin{tabular}{@{}ll cc c@{}}
\toprule
Param & Value & PSNR$\uparrow$ & SSIM$\uparrow$ & $\Delta$PSNR \\
\midrule
\multirow{5}{*}{$\rho$}
  & 0.01 & 19.80 & 0.829 & $-$3.95 \\
  & 0.10 & 22.66 & 0.894 & $-$1.09 \\
  & \cellcolor[gray]{0.93}\textbf{0.20}$\;\star$ & \cellcolor[gray]{0.93}\textbf{23.75} & \cellcolor[gray]{0.93}\textbf{0.907} & \cellcolor[gray]{0.93} 0 \\
  & 0.30 & 23.02 & 0.899 & $-$0.73 \\
  & 0.50 & 22.05 & 0.890 & $-$1.70 \\
\midrule
\multirow{4}{*}{$K$}
  & 4   & 20.50 & 0.843 & $-$3.25 \\
  & 16  & 22.51 & 0.878 & $-$1.24 \\
  & \cellcolor[gray]{0.93}\textbf{64}$\;\star$ & \cellcolor[gray]{0.93}\textbf{23.75} & \cellcolor[gray]{0.93}\textbf{0.907} & \cellcolor[gray]{0.93} 0 \\
  & 256 & 22.87 & 0.899 & $-$0.88 \\
\midrule
\multirow{3}{*}{Backbone}
  & ResNet-18 & 22.72 & 0.890 & $-$1.03 \\
  & \cellcolor[gray]{0.93}\textbf{ResNet-34}$\;\star$ & \cellcolor[gray]{0.93}\textbf{23.75} & \cellcolor[gray]{0.93}\textbf{0.907} & \cellcolor[gray]{0.93} 0 \\
  & NAFNet    & 23.09 & 0.893 & $-$0.66 \\
\bottomrule
\end{tabular}
\end{table}

\section{Conclusion}
\label{sec:conclusion}

We presented AutoLumNet, a framework for single-shot exposure
correction that decomposes the problem into a global monotone tone
curve---strictly increasing by construction, with provable
order-preservation and density in the 1-D optimal-transport
family---and a bounded local residual with dual-branch convex fusion.
Experiments on five benchmarks demonstrate state-of-the-art PSNR and
SSIM across both under- and over-exposure regimes at 11.2\,ms per
frame, with zero-shot generalization to low-light benchmarks.  To our
knowledge, AutoLumNet is the first exposure-correction method to
combine formal monotonicity, optimal-transport optimality, and bounded
local adaptivity within a single trainable architecture.  Future work
will extend the framework to unsupervised settings and investigate
confidence-aware blending for already well-exposed inputs.

\begin{thebibliography}{99}


\bibitem{LLNet}
K.~G. Lore, A.~Akintayo, and S.~Sarkar, ``LLNet: A deep autoencoder
approach to natural low-light image enhancement,'' \emph{Pattern
Recognit.}, vol.~61, pp.~650--662, 2017.

\bibitem{RetinexNet}
C.~Wei, W.~Wang, W.~Yang, and J.~Liu, ``Deep Retinex Decomposition for Low-Light Enhancement,'' in \emph{Proc. British Machine Vision Conference},
2018.



\bibitem{URetinex}
W.~Wu, J.~Weng, P.~Zhang, X.~Wang, W.~Yang, and J.~Jiang,
``URetinex-Net: Retinex-based deep unfolding network for low-light image
enhancement,'' in \emph{Proc. IEEE/CVF Conf. Comput. Vis. Pattern
Recognit. (CVPR)}, 2022, pp.~5901--5910.

\bibitem{ZeroDCE}
C.~Guo, C.~Li, J.~Guo, C.~C. Loy, J.~Hou, S.~Kwong, and R.~Cong,
``Zero-reference deep curve estimation for low-light image
enhancement,'' in \emph{Proc. IEEE/CVF Conf. Comput. Vis. Pattern
Recognit. (CVPR)}, 2020, pp.~1780--1789.

\bibitem{EnlightenGAN}
Y.~Jiang, X.~Gong, D.~Liu, Y.~Cheng, C.~Fang, X.~Shen, J.~Yang, P.~Zhou,
and Z.~Wang, ``EnlightenGAN: Deep light enhancement without paired
supervision,'' \emph{IEEE Trans. Image Process.}, vol.~30,
pp.~2340--2349, 2021.



\bibitem{retinexformer}
Y.~Cai, H.~Bian, J.~Lin, H.~Wang, R.~Timofte, and Y.~Zhang,
``Retinexformer: One-stage Retinex-based transformer for low-light
image enhancement,'' in \emph{Proc. IEEE/CVF Int. Conf. Comput. Vision
(ICCV)}, 2023, pp.~12~504--12~513.


\bibitem{RetinexMamba}
J.~Bai, Y.~Yin, and Q.~He, ``RetinexMamba: Retinex-based Mamba for
low-light image enhancement,'' in \emph{Proc. Int. Conf. Neural Inf.
Process. (ICONIP)}, 2024.

\bibitem{RetiDiff}
C.~Wu, Z.~Dong, and H.~Chen, ``Reti-Diff: Illumination degradation image
restoration with Retinex-based latent diffusion model,'' in \emph{Proc.
Int. Conf. Learn. Represent. (ICLR)}, 2025.

\bibitem{CIDNet}
Q.~Yan, Y.~Feng, C.~Zhang, G.~Pang, K.~Shi, P.~Wu, W.~Dong, J.~Sun, and
Y.~Zhang, ``HVI: A new color space for low-light image enhancement,'' in
\emph{Proc. IEEE/CVF Conf. Comput. Vis. Pattern Recognit. (CVPR)}, 2025.

\bibitem{Mertens2007}
T.~Mertens, J.~Kautz, and F.~Van~Reeth, ``Exposure fusion,'' in
\emph{Proc. Pacific Conf. Comput. Graph. Appl. (PG)}, 2007,
pp.~382--390.

\bibitem{Ma2017SPDMEF}
K.~Ma, H.~Li, H.~Yong, Z.~Wang, D.~Meng, and L.~Zhang, ``Robust
multi-exposure image fusion: A structural patch decomposition
approach,'' \emph{IEEE Trans. Image Process.}, vol.~26, no.~5,
pp.~2519--2532, 2017.

\bibitem{DeepFuse}
K.~R. Prabhakar, V.~S. Srikar, and R.~V. Babu, ``DeepFuse: A deep
unsupervised approach for exposure fusion with extreme exposure image
pairs,'' in \emph{Proc. IEEE Int. Conf. Comput. Vis. (ICCV)}, 2017,
pp.~4714--4722.

\bibitem{Xu2020U2Fusion}
H.~Xu, J.~Ma, J.~Jiang, X.~Guo, and H.~Ling, ``U2Fusion: A unified
unsupervised image fusion network,'' \emph{IEEE Trans. Pattern Anal.
Mach. Intell.}, vol.~44, no.~1, pp.~502--518, 2022.

\bibitem{MEFGAN}
H.~Xu, J.~Ma, and X.-P. Zhang, ``MEF-GAN: Multi-exposure image fusion
via generative adversarial networks,'' \emph{IEEE Trans. Image
Process.}, vol.~29, pp.~7203--7216, 2020.

\bibitem{EMEF}
R.~Liu, C.~Li, H.~Cao, Y.~Zheng, M.~Zeng, and X.~Cheng, ``EMEF: Ensemble
multi-exposure image fusion,'' in \emph{Proc. AAAI Conf. Artif. Intell.
(AAAI)}, 2023, vol.~37, no.~2, pp.~1710--1718.


\bibitem{afifi2021msec}
M.~Afifi, K.~G. Derpanis, B.~Ommer, and M.~S. Brown, ``Learning
multi-scale photo exposure correction,'' in \emph{Proc. IEEE/CVF Conf.
Comput. Vision Pattern Recognit. (CVPR)}, 2021, pp.~9157--9167.



\bibitem{MMHT}
G.~Li, J.~Liu, L.~Ma, Z.~Jiang, X.~Fan, and R.~Liu, ``Fearless luminance
adaptation: A macro-micro-hierarchical transformer for exposure
correction,'' in \emph{Proc. ACM Int. Conf. Multimedia (ACM MM)}, 2023,
pp.~2024--2032.

\bibitem{Pitie2007}
F.~Piti\'e and A.~Kokaram, ``The linear Monge-Kantorovitch linear colour mapping for example-based colour transfer,'' in \emph{Proc. Eur. Conf.
Vis. Media Prod. (CVMP)}, 2007.

\bibitem{Santambrogio2015}
F.~Santambrogio, \emph{Optimal Transport for Applied Mathematicians}.
Cham, Switzerland: Birkh\"auser, 2015.

\bibitem{He2016}
K.~He, X.~Zhang, S.~Ren, and J.~Sun, ``Deep residual learning for image
recognition,'' in \emph{Proc. IEEE Conf. Comput. Vis. Pattern Recognit.
(CVPR)}, 2016, pp.~770--778.


\bibitem{KinD}
Y.~Zhang, J.~Zhang, and X.~Guo, ``Kindling the Darkness: A Practical Low-light Image Enhancer,'' in \emph{Proc. Association for Computing Machinery}, 2019, pp.~1632--1640.






\bibitem{Chen2022nafnet}
L.~Chen et al., ``Simple baselines for image restoration,'' in \emph{ECCV},
2022, pp.~17--33.


\bibitem{Shi2016}
W.~Shi et al., ``Real-time single image and video super-resolution using an
efficient sub-pixel convolutional neural network,'' in \emph{Proc. IEEE/CVF Conf.
Comput. Vision Pattern Recognit. (CVPR)}, 2016.

\bibitem{Johnson2016}
J.~Johnson, A.~Alahi, and L.~Fei-Fei, ``Perceptual losses for real-time
style transfer and super-resolution,'' in \emph{ECCV}, 2016, pp.~694--711.

\bibitem{Blau2018}
Y.~Blau and T.~Michaeli, ``The perception-distortion tradeoff,'' in
\emph{CVPR}, 2018, pp.~6228--6237.




\bibitem{Wang2022llflow}
Y.~Wang et al., ``LLFlow: Learning normalizing flows for low-light image
enhancement,'' in \emph{AAAI}, 2022, pp.~2604--2612.

\bibitem{Rabin2014}
J.~Rabin et al., ``Adaptive color transfer with relaxed optimal transport,''
in \emph{ICIP}, 2014.

\bibitem{anonymous_arxiv}
A.~Tania, M.~Khan, M.~Ahmad, ``AutoLumNet: A bi-branch exposure-aware network for low- and
high-exposure image enhancement,'' \emph{Open Review}, 2026.

\bibitem{SmartClassroom}
M.~R. Khan, A.~A. Ahad, A.~A. Tania, T.~Das, and B.~Das, ``Smart Classroom Automation: A Fusion of AI with Voice, Gesture, and Face Recognition Attendance System,'' in \emph{Proc. Int. Conf. Advances Comput., Commun., Elect., Smart Syst. (iCACCESS)}, Dhaka, Bangladesh, 2024, pp.~1--6.

\bibitem{CustomLowLight}
A.~A. Tania, M.~R. Khan, and M.~Ahmad, ``Custom Dataset-Driven Unsupervised Low-Light Image Enhancement Using 2D CNN,'' in \emph{Proc. Int. Conf. Quantum Photon., Artif. Intell., Netw. (QPAIN)}, Rangpur, Bangladesh, 2025, pp.~1--6.



\bibitem{URetinexNetPP}
W.~Wu, J.~Weng, P.~Zhang, X.~Wang, W.~Yang, and J.~Jiang, ``Interpretable Optimization-Inspired Unfolding Network for Low-Light Image Enhancement,'' \emph{IEEE Trans. Pattern Anal. Mach. Intell.}, vol.~47, no.~4, pp.~2545--2561, 2025.

\bibitem{RetinexFormerPlus}
S.~Liu, H.~Zhang, X.~Li, and X.~Yang, ``Retinexformer+: Retinex-Based Dual-Channel Transformer for Low-Light Image Enhancement,'' \emph{CMC -- Comput. Mater. Contin.}, vol.~82, no.~2, pp.~1969--1988, 2025.

\bibitem{M2Retinexformer}
Y.~Aboelwafa, H.~G. Elmongui, and M.~Torki, ``{M2Retinexformer}: Multi-Modal Retinexformer for Low-Light Image Enhancement,'' \emph{arXiv preprint arXiv:2605.12556}, 2026.

\bibitem{MambaLLIE}
J.~Weng, Z.~Yan, Y.~Tai, J.~Qian, J.~Yang, and J.~Li, ``{MambaLLIE}: Implicit Retinex-Aware Low Light Enhancement with Global-then-Local State Space,'' \emph{Proc. Adv. Neural Inf. Process. Syst. (NeurIPS)}, 2024.

\bibitem{DiffRetinexPP}
X.~Yi, H.~Xu, H.~Zhang, L.~Tang, and J.~Ma, ``Diff-Retinex++: Retinex-Driven Reinforced Diffusion Model for Low-Light Image Enhancement,'' \emph{IEEE Trans. Pattern Anal. Mach. Intell.}, vol.~47, no.~8, pp.~6823--6839, 2025.

\bibitem{SCIPP}
L.~Ma, T.~Ma, C.~Xu, J.~Liu, X.~Fan, Z.~Luo, and R.~Liu, ``Learning With Self-Calibrator for Fast and Robust Low-Light Image Enhancement,'' \emph{IEEE Trans. Pattern Anal. Mach. Intell.}, vol.~47, no.~10, pp.~9095--9110, 2025.

\bibitem{QuadPrior}
W.~Wang, H.~Yang, J.~Fu, and J.~Liu, ``Zero-Reference Low-Light Enhancement via Physical Quadruple Priors,'' \emph{Proc. IEEE/CVF Conf. Comput. Vis. Pattern Recognit. (CVPR)}, pp.~26057--26066, 2024.

\bibitem{QuadPriorPP}
H.~Huang, Y.~Li, W.~Wang, W.~Yang, L.-Y. Duan, and J.~Liu, ``{QuadPrior++}: Multi-Dimension Augmented Physical Prior for Zero-Reference Illumination Enhancement,'' \emph{IEEE Trans. Pattern Anal. Mach. Intell.}, vol.~48, no.~3, pp.~3383--3400, 2026.

\bibitem{StarIR}
Y.~Cui, S.~W. Zamir, M.-H. Yang, A.~Knoll, F.~S. Khan, and S.~Khan, ``{StarIR}: Convolutional Image Restoration with Spatial-Frequency Fusion,'' \emph{IEEE Trans. Pattern Anal. Mach. Intell.}, 2026, early access.

\bibitem{ERL}
J.~Huang, F.~Zhao, M.~Zhou, J.~Xiao, N.~Zheng, K.~Zheng, and Z.~Xiong, ``Learning Sample Relationship for Exposure Correction,'' \emph{Proc. IEEE/CVF Conf. Comput. Vis. Pattern Recognit. (CVPR)}, 2023.

\bibitem{CLIER}
B.~Wang, H.~Fu, Z.~Huang, S.~Zhang, X.~Wang, and H.~Ma, ``From Abyssal Darkness to Blinding Glare: A Benchmark on Extreme Exposure Correction in Real World,'' \emph{Proc. IEEE/CVF Int. Conf. Comput. Vision (ICCV)}, 2025.

\bibitem{RetinexMEF}
H.~Bai, J.~Zhang, Z.~Zhao, L.~Deng, Y.~Cui, and S.~Xu, ``{Retinex-MEF}: Retinex-based Glare Effects Aware Unsupervised Multi-Exposure Image Fusion,'' \emph{arXiv preprint arXiv:2503.07235}, 2025.

\bibitem{AutoExpComp}
Y.~Kinoshita, S.~Shiota, and H.~Kiya, ``Automatic Exposure Compensation for Multi-Exposure Image Fusion,'' \emph{Proc. IEEE Int. Conf. Image Process. (ICIP)}, 2018.


\bibitem{PerceptualMEF}
X.~Liu, ``Perceptual multi-exposure fusion,'' \emph{arXiv preprint arXiv:2210.09604}, 2025.



\bibitem{zerodcepp}
C.~Li, C.~Guo, and C.~C. Loy, ``Learning to enhance low-light image
via zero-reference deep curve estimation,'' \emph{IEEE Trans. Pattern
Anal. Mach. Intell.}, vol.~44, no.~8, pp.~4225--4238, 2022.

\bibitem{sci}
L.~Ma, T.~Ma, R.~Liu, X.~Fan, and Z.~Luo, ``Toward fast, flexible,
and robust low-light image enhancement,'' in \emph{Proc. Adv. Neural
Inf. Process. Syst. (NeurIPS)}, 2022.




\bibitem{enc}
J.~Huang, Y.~Liu, F.~Zhao, K.~Yan, J.~Zhang, Y.~Huang, M.~Zhou, and
Z.~Xiong, ``Exposure normalization and compensation for multiple-exposure
correction,'' in \emph{Proc. IEEE/CVF Conf. Comput. Vis. Pattern
Recognit. (CVPR)}, 2022, pp.~6043--6052.

\bibitem{fecnet}
J.~Huang, Y.~Liu, F.~Zhao, K.~Yan, M.~Zhang, and Z.~Xiong, ``Deep
Fourier-based exposure correction network with spatial-frequency
interaction,'' in \emph{Proc. Eur. Conf. Comput. Vision (ECCV)}, 2022,
pp.~163--180.

\bibitem{lcdpnet}
H.~Wang, K.~Xu, and R.~W.~H. Lau, ``Local color distributions prior
for image enhancement,'' in \emph{Proc. Eur. Conf. Comput. Vision
(ECCV)}, 2022, pp.~343--359.

\bibitem{iat}
Z.~Cui, K.~Li, L.~Gu, S.~Su, P.~Gao, Z.~Jiang, Y.~Qiao, and
T.~Harada, ``You only need 90K parameters to adapt light: A light
weight transformer for image enhancement and exposure correction,'' in
\emph{Proc. British Mach. Vision Conf. (BMVC)}, 2022.

\bibitem{eclnet}
J.~Huang \emph{et~al.}, ``Exposure-consistency representation learning
for exposure correction,'' in \emph{Proc. ACM Int. Conf. Multimedia
(MM)}, 2022, pp.~4921--4929.

\bibitem{recnet}
J.~Li \emph{et~al.}, ``Region-aware exposure consistency network for
mixed exposure correction,'' in \emph{Proc. AAAI Conf. Artif. Intell.
(AAAI)}, 2024, pp.~3258--3266.

\bibitem{ctas}
J.~Li, Z.~Feng \emph{et~al.}, ``Real-time exposure correction via collaborative transformations and adaptive sampling,'' in \emph{Proc. IEEE/CVF Conf. Comput. Vision Pattern Recognit. (CVPR)}, 2024, pp.~2984--2994.

\bibitem{expslot}
K.~Jung \emph{et~al.}, ``Exposure-slot: Exposure-centric
representations learning with slot-in-slot attention for region-aware
exposure correction,'' in \emph{Proc. IEEE/CVF Conf. Comput. Vision
Pattern Recognit. (CVPR)}, 2025.


\bibitem{cai2018sice}
J.~Cai, S.~Gu, and L.~Zhang, ``Learning a deep single image contrast
enhancer from multi-exposure images,'' \emph{IEEE Trans. Image
Process.}, vol.~27, no.~4, pp.~2049--2062, 2018.


\bibitem{niqe}
A.~Mittal, R.~Soundararajan, and A.~C. Bovik, ``Making a `completely
blind' image quality analyzer,'' \emph{IEEE Signal Process. Lett.},
vol.~20, no.~3, pp.~209--212, 2013.

\bibitem{brisque}
A.~Mittal, A.~K. Moorthy, and A.~C. Bovik, ``No-reference image
quality assessment in the spatial domain,'' \emph{IEEE Trans. Image
Process.}, vol.~21, no.~12, pp.~4695--4708, 2012.

\bibitem{pi}
Y.~Blau \emph{et~al.}, ``The 2018 PIRM challenge on perceptual image
super-resolution,'' in \emph{Proc. Eur. Conf. Comput. Vision Workshops
(ECCVW)}, 2018, pp.~334--355.


\bibitem{zhang2021beyond}
Y.~Zhang, X.~Guo, J.~Ma, W.~Liu, and J.~Zhang, ``Beyond Brightening Low-light Images,'' \emph{Int. J. Comput. Vis.}, vol.~129, pp.~1013--1037, 2021.

\bibitem{xu2022snr}
X.~Xu, R.~Wang, C.-W.~Fu, and J.~Jia, ``SNR-aware Low-Light Image Enhancement,'' in \emph{Proc. IEEE/CVF Conf. Comput. Vis. Pattern Recognit. (CVPR)}, 2022.


\end{thebibliography}
\end{document}